\documentclass{article}

\usepackage{microtype}
\usepackage{graphicx}
\usepackage{booktabs} % for professional tables

\usepackage{hyperref}

\usepackage[preprint]{icml2025}
\usepackage{caption}

\usepackage{amsmath}
\usepackage{amssymb}
\usepackage{mathtools}
\usepackage{amsthm}

\usepackage[capitalize,noabbrev]{cleveref}

\theoremstyle{plain}
\newtheorem{theorem}{Theorem}[section]

\theoremstyle{definition}
\newtheorem{definition}[theorem]{Definition}

\theoremstyle{remark}

\usepackage[textsize=tiny]{todonotes}

\usepackage{subcaption}
\usepackage{dsfont} % 
\usepackage{calc}
\usepackage{comment}

\icmltitlerunning{Toward Architecture-Agnostic Local Control of Posterior Collapse in VAEs}

\begin{document}

\twocolumn[
\icmltitle{Toward Architecture-Agnostic Local Control of Posterior Collapse in VAEs}

% It is OKAY to include author information, even for blind
% submissions: the style file will automatically remove it for you
% unless you've provided the [accepted] option to the icml2024
% package.

% List of affiliations: The first argument should be a (short)
% identifier you will use later to specify author affiliations
% Academic affiliations should list Department, University, City, Region, Country
% Industry affiliations should list Company, City, Region, Country

% You can specify symbols, otherwise they are numbered in order.
% Ideally, you should not use this facility. Affiliations will be numbered
% in order of appearance and this is the preferred way.
%\icmlsetsymbol{equal}{*}

\begin{icmlauthorlist}
\icmlauthor{Hyunsoo Song}{nims}
\icmlauthor{Seungwhan Kim}{snu}
\icmlauthor{Seungkyu Lee}{khu}
\end{icmlauthorlist}

\icmlaffiliation{nims}{Innovation Center for Industrial Mathematics, National Institute for Mathematical Sciences, Seongnam-si, Republic of Korea}
\icmlaffiliation{snu}{Interdisciplinary Program in Artificial Intelligence, Seoul National University, Seoul-si, Republic of Korea}
\icmlaffiliation{khu}{Department of Computer Science and Engineering, KyungHee University, Yongin-si, Republic of Korea}

\icmlcorrespondingauthor{Hyunsoo Song}{song@nims.re.kr}
\icmlcorrespondingauthor{Seungkyu Lee}{seungkyu@khu.ac.kr}
%\icmlcorrespondingauthor{}{}

% You may provide any keywords that you
% find helpful for describing your paper; these are used to populate
% the "keywords" metadata in the PDF but will not be shown in the document
\icmlkeywords{VAE, Posterior collapse, Generative model, Latent model}

\vskip 0.3in
]

% this must go after the closing bracket ] following \twocolumn[ ...

% This command actually creates the footnote in the first column
% listing the affiliations and the copyright notice.
% The command takes one argument, which is text to display at the start of the footnote.
% The \icmlEqualContribution command is standard text for equal contribution.
% Remove it (just {}) if you do not need this facility.

\printAffiliationsAndNotice{}  % leave blank if no need to mention equal contribution
%\printAffiliationsAndNotice{\icmlEqualContribution} % otherwise use the standard text.

\begin{abstract}
Variational autoencoders (VAEs), one of the most widely used generative models, are known to suffer from posterior collapse, a phenomenon that reduces the diversity of generated samples. % 문제 제기
To avoid posterior collapse, many prior works have tried to control the influence of regularization loss. However, the trade-off between reconstruction and regularization is not satisfactory. For this reason, several methods have been proposed to guarantee latent identifiability, which is the key to avoiding posterior collapse. However, they require structural constraints on the network architecture. % 기존 연구들과 그 한계
For further clarification, we define local posterior collapse to reflect the importance of individual sample points in the data space and to relax the network constraint. % 문제 정의
Then, we propose \textit{Latent Reconstruction}(LR) loss, which is inspired by mathematical properties of injective and composite functions, to control posterior collapse without restriction to a specific architecture. % 제안하는 메소드
We experimentally evaluate our approach, which controls posterior collapse on varied datasets such as MNIST, fashionMNIST, Omniglot, CelebA, and FFHQ. % 평가
\end{abstract}

\section{Introduction}
% Introduction:
% v 0. 기존 메소드들과 그 한계
% 0. amortization gap과 posterior collapse 간의 상관관계 -> 기존 논문들 인용해서 설명하는 것으로 충분할 듯

VAEs(Variational auto-encoders) are popular generative latent models, which have been used in various applications such as representation learning, manifold learning, and dimension reduction. However, they have suffered from a chronic posterior collapse problem, which reduces the diversity of outputs. Although many studies\cite{lucas2019understanding}\cite{lucas2019don}\cite{dai2020usual}\cite{razavi2019preventing}\cite{clapham2023posterior}\cite{rueckert2023cr}\cite{wang2021posterior}\cite{kinoshita2023controlling} have tried to analyze and address it, some ambiguity and limitations remain. % 전체적인 문제 제기
Regularization loss is highly related to the posterior collapse phenomenon. Thus, tuning the importance($\beta$)\cite{higgins2017beta} of regularization loss has been widely used. From a theoretical perspective, many works, such as \cite{ichikawa2024learning} report that posterior collapse is unavoidable once $\beta$ exceeds a specific value. However, the optimal point on the trade-off curve provided by these methods has been regarded as unsatisfactory. Minimizing the ELBO objective in VAEs does not necessarily ensure both accurate latent inference and faithful data reconstruction\cite{zhao2019infovae}\cite{dai2021value}. A recent study\cite{song2024scalevae} similarly argues that posterior collapse arises from optimization dynamics rather than solely from KL imbalance. % 기존 주요 접근법 1(ELBO 조정)과 그의 한계점
A study\cite{wang2021posterior} has shown that preserving the identifiability of the latent variable, which is key to avoiding posterior collapse, can be achieved by implementing an injective decoder based on a \textit{Brenier map}. Ensuring latent variable identifiability makes the ELBO optimum more reliable. However, they only avoid the extreme case of posterior collapse, and the decoder architecture is restricted to ICNN\cite{amos2017input}. Therefore, later work\cite{kinoshita2023controlling}\cite{kinoshita2024provable} explores the control of posterior collapse under more general conditions. However, the global Lipschitz condition may be strong for ideal latent models. In addition, the choice of architecture is still restricted to be a \textit{Brenier map}. % 기존 주요 접근법2(네트워크 제약)와 그것의 한계점

In this work, we define posterior collapse in more general and relaxed manner. In addition, we show that ensuring \textit{locally $L(z)$-bi-Lipschitz continuity} of the encoder or decoder helps to avoid posterior collapse. To realize this condition in an architecture-agnostic manner, we propose LRVAE (Latent Reconstruction VAE), which is a loss-based approach that implicitly encourages latent identifiability through a \textit{symmetric reconstruction loss} composed of data reconstruction and latent reconstruction terms. Our method remains broadly applicable, offering high practical utility. We also evaluate our models, which are implemented based on both shallow and deep~\cite{child2020very} VAE architectures as the backbone, with measuring metrics related to posterior collapse and showing responses w.r.t. tiny latent vector differences. In our experiments, MNIST~\cite{deng2012mnist}, fashionMNIST~\cite{xiao2017fashion}, Omniglot~\cite{lake2015human}, CelebA~\cite{liu2015deep}, and FFHQ-256~\cite{karras2019style} datasets are used. Experimental results show that VAEs effectively avoid posterior collapse, securing generation output diversity even at around tiny latent differences.

\section{Local Posterior Collapse and Local bi-Lipschitz Continuity} % ## 서브섹션 나눠야 할지? 
\label{sec:pc}

% posterior collapse 기존 정의들 %
A pioneer work\cite{wang2021posterior} defines posterior collapse as the phenomenon in which the learned posterior becomes non-informative or input-independent, leading to a loss of latent identifiability. They define posterior collapse more strictly as the equality $p_{\hat{\theta}}(z|x) = p(z)$, where $p_{\hat{\theta}}(z|x)$ is the true posterior under the optimized parameter $\hat{\theta}$. However, this definition may be overly strict as they do not tolerate even small input-dependent variations in the posterior distribution. For this reason, a later work\cite{kinoshita2023controlling} uses the $\epsilon$-posterior collapse as $d(q_{\hat{\varphi}}(z|x), p(z)) \leq \epsilon$ to relax the condition. However, a globally applied constant threshold $\epsilon$ might still impose an overly restrictive condition. Therefore, relaxation would be required. For example, a previous work\cite{lucas2019don} defines $(\epsilon, \Delta)$-posterior collapse as a probabilistic representation $\mathbb{P}_{x\sim P_{\mathcal{X}}(x)}[D_{KL}(q_{\hat{\varphi}}(z|x), p(z)) < \epsilon] \geq 1-\Delta$. Another work\cite{razavi2019preventing} loosely defines posterior collapse as the KL divergence $D_{KL}(q(z|x)\|p(z))$ approaches zero, where $q(z|x)$ is the approximate posterior and $p(z)$ is the prior. We summarize the posterior collapse definitions in Appendix~\ref{app:vary_pc}. 

% e(x)-PC 필요성 %
Despite many definitions having been proposed, some ambiguities remain. Empirical datasets rarely exhibit uniform conditional complexity: some inputs \(x\) can be reconstructed with negligible latent information, whereas others require a substantial mutual‑information budget. A single global tolerance \(\epsilon\) for detecting posterior collapse
\(\mathrm{KL}\!\bigl[q_\varphi(z\mid x)\,\|\,p(z)\bigr]\le\epsilon\) is therefore ill‑suited: it \emph{under‑detects} collapse in high‑entropy regions and \emph{over‑penalises} low‑entropy ones. To align the collapse test with this heterogeneous geometry, we introduce an input‑dependent threshold \(\epsilon:x\!\to\![0,\infty)\).
The resulting \(\epsilon(x)\)-posterior‑collapse criterion (i) preserves information where the decoder is sensitive, (ii) allows stronger compression where it is benign, and (iii) integrates naturally with the local bi‑Lipschitz analysis developed in the sequel.

\begin{definition}[$\epsilon(x)$-Posterior Collapse]\label{def:epsx_posterior_collapse}
Let $X = \{x_1, \dots, x_n\} \subset \mathbb{R}^D$ denote a finite dataset, and let $\mathcal{X} \subset \mathbb{R}^D$ be a compact domain that includes the support of the empirical data distribution.  
We assume the data distribution $p_X$ as the uniform mixture of Dirac delta measures centered at the data points, i.e.,  
$p_X(x) := \frac{1}{n} \sum_{i=1}^n \delta_{x_i}(x)$.

We say that a model exhibits \emph{$\epsilon(x)$-posterior collapse} if, for both true posterior $p_{\hat{\theta}}(z|x)$ and approximate posterior $q_{\hat{\varphi}}(z|x)$, the following holds:
\begin{align}
\sup_{\rho\in\{p_{\hat\theta},\,q_{\hat\varphi}\}}
d\bigl(\rho(z|x),\,p(z)\bigr)
\;\le\;\epsilon(x),
\quad\forall x\in\mathcal{X}.
\end{align}
where $d(\cdot,\cdot)$ is a divergence (e.g., KL divergence or Wasserstein distance), $\hat{\theta}$ and $\hat{\varphi}$ denote optimized parameters of decoder and encoder, and $\epsilon: \mathcal{X} \to [0,\infty)$ is a threshold function satisfying:
\begin{enumerate}
    \item (\textbf{Maximum on data}) There exists a constant $\epsilon_{\max} > 0$ such that
    \begin{equation*}
        \epsilon(x) = \epsilon_{\max} \quad \text{for all } x \in X.
    \end{equation*}

    \item (\textbf{Monotonic decay}) For all $x, x' \in \mathcal{X}$,
    \begin{equation*}
        \mathrm{dist}(x, X) < \mathrm{dist}(x', X) \quad \Rightarrow \quad \epsilon(x) \geq \epsilon(x').
    \end{equation*}

    \item (\textbf{Asymptotic decay}) As $x$ moves away from the dataset,
    \begin{equation*}
        \lim_{\mathrm{dist}(x, X) \to \infty} \epsilon(x) = 0.
    \end{equation*}
\end{enumerate}
Here, $\mathrm{dist}(x, X) := \min_{1 \leq i \leq n} \|x - x_i\|$ denotes the Euclidean distance to the dataset. This definition (\ref{def:epsx_posterior_collapse}) generalizes the notion of $\epsilon$-posterior collapse~\cite{kinoshita2023controlling} by introducing a pointwise threshold function $\epsilon(x)$ instead of a constant.  
This functional form allows greater flexibility by tolerating small deviations from the prior in less informative regions of the input space (e.g., noisy or unstructured images), thereby avoiding overly restrictive constraints on the model behavior.  
In contrast, previous works adopted a constant $\epsilon$ determined via heuristic tuning.  
However, such global constants are inherently ambiguous, as the optimal threshold depends on the model capacity, data distribution, and the specific notion of collapse used.  
Therefore, in this work, we do not assume any explicit or fixed form for $\epsilon(x)$, allowing it to reflect the local informativeness of each input. %
The approximate posterior collapse can also be defined in the same manner(see Appendix \ref{app:vary_pc}). In our work, the true posterior and the approximate posterior converge to the same distribution by using symmetric reconstruction loss, as shown in \ref{thr:true_approx_conv}. 
In addition, we define the posterior collapse for both approximate posterior and true posterior. Because the convergence of the encoder and decoder is interdependent; if either collapses, they cannot be trained correctly. The inference gap, which is the difference between these two posteriors, is intractable but must be minimized to justify the existence of the encoder. This implies that both types of posterior collapse must be avoided.
\end{definition}

% ### TODO: theorem 본문에 있는 것은 설명을 쉽게 표현하기. 현재는 너무 복잡해보임.

% ### locally L(x) bi-Lipschitz 
\begin{definition}[$L(z)$-bi-Lipschitz continuity of decoder with probability $1-\zeta$] \label{def:L-bi-Lip}
Let $D_\theta\!:\mathcal Z\!\to\!\mathcal X$ be the decoder and let $p(z)$ be a prior on $\mathcal Z$.
For any confidence level(a scalar) $\zeta\in(0,1)$ define the radius $R$ as:
\begin{equation*}
  R_\zeta \;:=\; \inf\!\bigl\{R>0 \;\bigl|\; \Pr_{Z\sim p}\!\bigl(\|Z\|\le R\bigr)\,\ge\,1-\zeta \bigr\}.
\end{equation*}
We say that $D_\theta$ is \emph{$L(z)$-bi-Lipschitz with probability $1-\zeta$} if there exist measurable functions
\[
  r:\mathcal Z\to(0,\infty), \qquad
  L:\mathcal Z\to[1,\infty)
\]
such that for every $z\in\mathcal Z$ with $\|z\|\le R_\zeta$ and for all $z_1,z_2\in B\!\bigl(z,r(z)\bigr)$,
\begin{equation}
  \frac{1}{L(z)}\|z_1-z_2\|
  \;\le\;
  \bigl\|D_\theta(z_1)-D_\theta(z_2)\bigr\|
  \;\le\;
  L(z)\,\|z_1-z_2\|.
\end{equation}
Moreover, the following global bounds (within the effective support) are finite and strictly positive:
\begin{align*}
  r_{\min}^{(\zeta)}
    \;:=\;\inf_{\|z\|\le R_\zeta} r(z)
    \;>\;0, \\
  L_{\max}^{(\zeta)}
    \;:=\;\sup_{\|z\|\le R_\zeta} L(z)
    \;<\;\infty.
\end{align*}
By this definition, we represent the bi-Lipschitz continuity locally.
\end{definition}

% Local L bi-Lipschitz => avoid PC
\begin{theorem}[Gaussian VAE: $L(z)$-bi-Lipschitz avoids $\epsilon(x)$-posterior collapse with probability $1-\zeta$]\label{thm:LbL_avoid_LPC}
Fix a confidence level $\zeta\in(0,1)$ and let radius $R_\zeta$ be defined as in Definition \ref{def:L-bi-Lip}.  
Let the prior be $p(z)=\mathcal N(0,I)$ and the likelihood
\[
  p(x\mid z)=\mathcal N\!\bigl(D_{\hat\theta}(z),\,\sigma^2 I\bigr).
\]
Assume the decoder $D_{\hat\theta}$ is $\mathcal C^2$ and \emph{$L(z)$-bi-Lipschitz with probability $1-\zeta$}.  
Consider any $x\in\mathcal X$ for which there exists $z_0$ satisfying
\[
  \|z_0\|\le R_\zeta, 
  \quad D_{\hat\theta}(z_0)=x.
\]
Define
\begin{align*}
  \Sigma &\;=\; \Bigl(I + \tfrac{1}{\sigma^2}\,J_D(z_0)^{\!\top} J_D(z_0)\Bigr)^{-1}, \\
  \kappa_{\min} &\;=\; \sigma_{\min}\bigl(J_D(z_0)\bigr)\;\ge\;\tfrac{1}{L(z_0)} \;>\;0.
\end{align*}
Then, under the Laplace approximation,
\begin{align*}
  &p(z\mid x)\approx \mathcal N\bigl(z_0,\Sigma\bigr) \\
  &\Rightarrow
  D_{\mathrm{KL}}\!\bigl(p(z\mid x)\,\|\,p(z)\bigr)
  \;\ge\;\frac{C}{2}\,
           \log\!\Bigl(1+\tfrac{\kappa_{\min}^{2}}{\sigma^{2}}\Bigr)
  \;>\;0.
\end{align*}
Consequently, with probability at least $1-\zeta$ over $x\sim p_{\mathrm{data}}$,
no $\epsilon(x)<\tfrac{C}{2}\log\!\bigl(1+\kappa_{\min}^{2}/\sigma^{2}\bigr)$
can satisfy Definition~\ref{def:epsx_posterior_collapse}. The proof of this Theorem is provided in Appendix \ref{app:proof_LbL_avoid_LPC}.
\end{theorem}

\section{Method: Latent Reconstruction}

\begin{figure}
    \centering
    \includegraphics[width=0.85\linewidth]{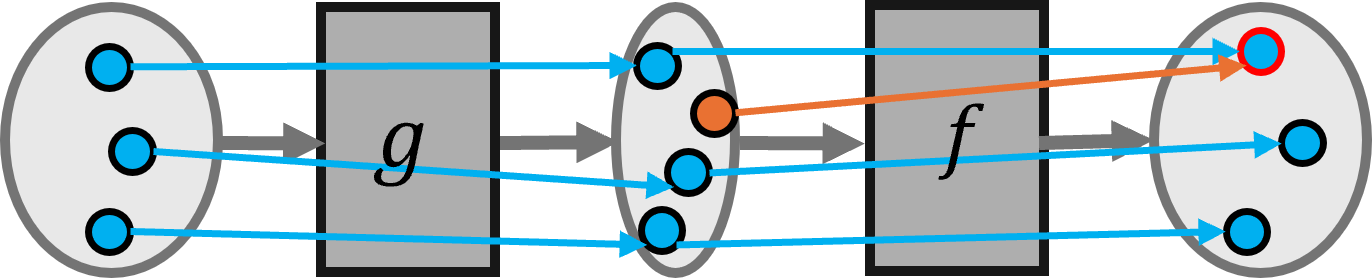}
    \caption{A sample injective/composite function where $g$ is injective but $f$ is not when $f\cdot g$ is injective: The orange dot on latent space has same output with other latent vector.}
    \label{fig:composite_function}
\end{figure}

We propose the Latent Reconstruction(LR) loss ($\mathcal{L}_{LR}$) as an additional loss term for ELBO objectives as follows:
\begin{equation*}
    \mathcal{L}_{LR} = \mathbb{E}_{p_{\theta}(x|z)}[-\log q_{\varphi}(z|x)] = \| {E_{\varphi, \mu}(D_{\theta}(z)) - z \|}^2,
\end{equation*}
where $q_{\varphi}(z|x)$ is assumed to be a normal distribution, and $E_{\varphi, \mu}$ denotes the encoder function that has an output $\mu$(mean of $z$) before latent random reparameterization. 

If a composite function $f\cdot g$ is injective, then $g$ must be injective, but $f$ is not necessarily so. Figure \ref{fig:composite_function} represents an example where $g$ is an injective but $f$ is not, even though $f\cdot g$ is injective. We apply this mathematical property to the VAE architecture.

%---------------------------------------------
% Key idea: LR pushes E o D ≈ Id_z  ⇒ beyond injective ⇒ local L(z) bi-Lipschitz (shown later)
%---------------------------------------------
Our key idea is to \emph{stochastically} enforce that the composite map \(E_{\varphi,\mu} \circ D_{\theta}\) approximates the identity on latent draws, i.e.,
\begin{equation*}\label{eq:lr-identity}
E_{\varphi,\mu}\!\bigl(D_{\theta}(z)\bigr) \approx z.
\end{equation*}
Driving down the latent reconstruction (LR) loss tightens this identity approximation, pressuring \(D_{\theta}\) to admit a data-driven inverse realized by \(E_{\varphi,\mu}\). 

In particular, DR and LR loss minimization promotes \textbf{relaxed latent identifiability that goes beyond mere injectivity}. While precisely defining this resulting state is challenging, we interpret it as locally($L(z)$) bi-Lipschitz behavior. This is visualized in Figure \ref{fig:elbo}.

\begin{figure}
    \centering
    \begin{subfigure}{0.48\columnwidth}
        \includegraphics[width=\textwidth]{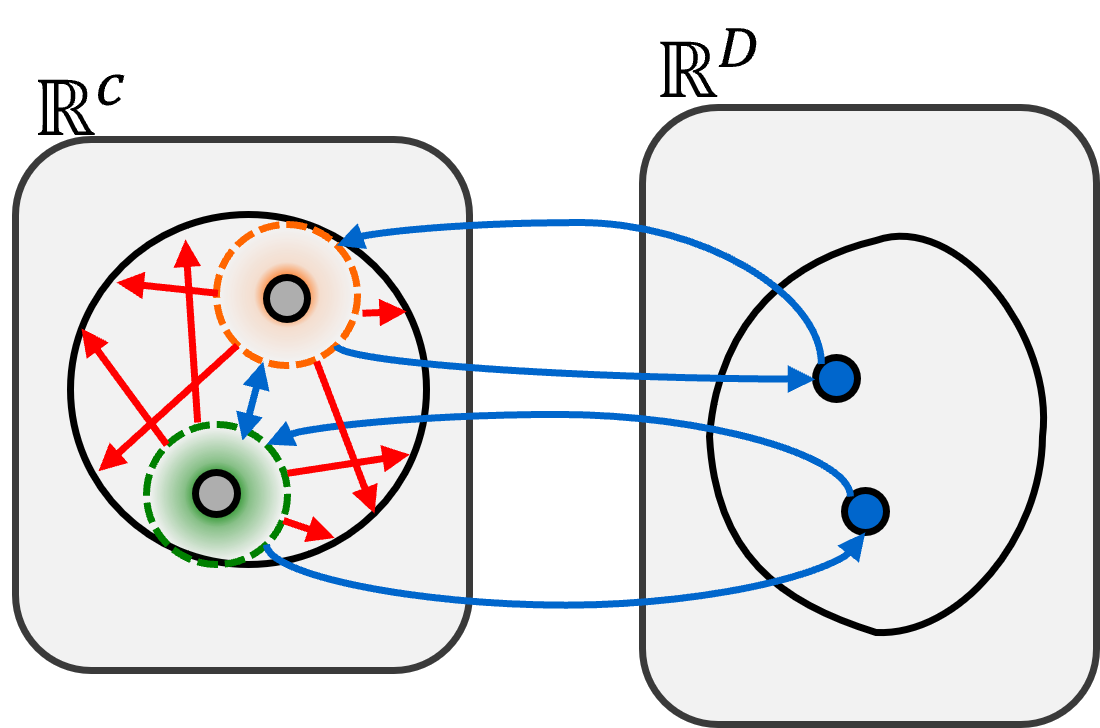}
        \caption{VAE}
        \label{fig:elbo_vae}
    \end{subfigure}
    \hfill
    \begin{subfigure}{0.48\columnwidth}
        \includegraphics[width=\textwidth]{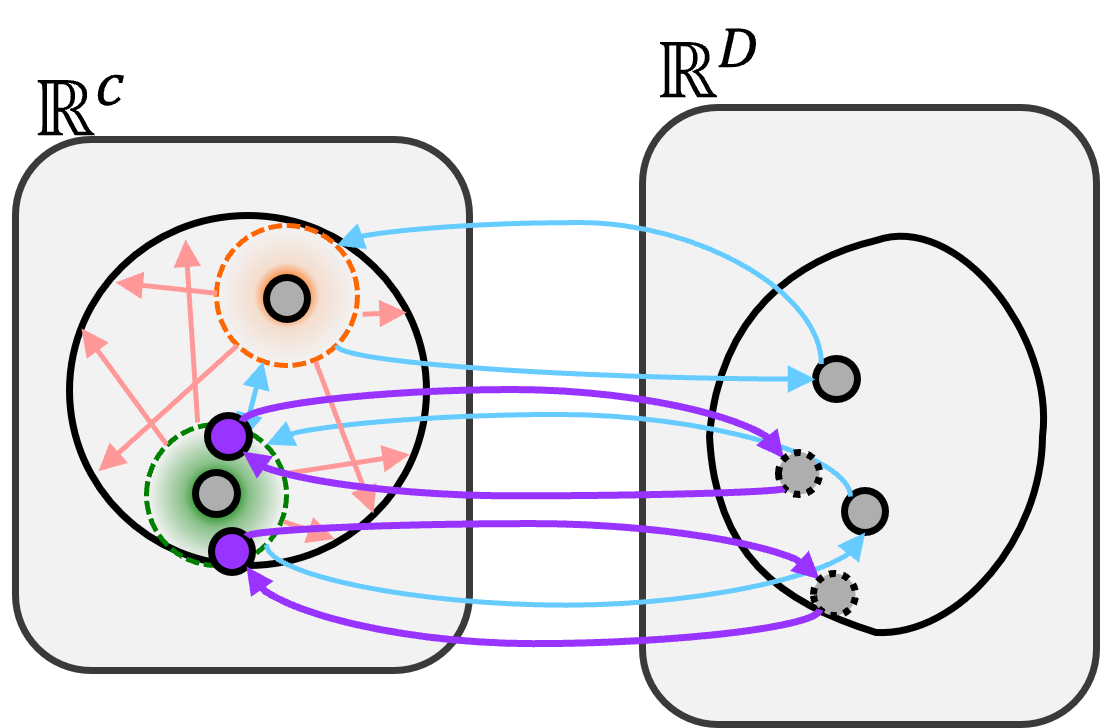}
        \caption{LRVAE}
        \label{fig:elbo_lrvae}
    \end{subfigure}
    \caption{Influence of ELBO objectives of VAE and our LRVAE in latent space($\mathbb{R}^C$) and data space($\mathbb{R}^D$): Regularization loss(Red arrows) leads latent variables to be close to prior. Reconstruction loss(Blue Arrows) leads latent variables identifiable from each other against regularization loss. However, reconstruction loss does not lead latent point to be identifiable from its distribution. The reconstruction loss secure latent identifiability only to distinguish data point. Whereas, Our latent reconstruction loss(Violet arrows), which is an additional term, directly secures identifiability of whole latent variable.}
    \label{fig:elbo}
\end{figure} % ### 이거 ref 하고 설명하기!

\begin{theorem}[High–Probability Locally Bi-Lipschitz under Piecewise-Linear Networks]\label{thm:hp-bilip-relaxed}
Fix confidence levels $\zeta,\varrho\in(0,1)$ and let radius $R_\zeta$ be defined as in Definition \ref{def:L-bi-Lip}.  
Let $D_\theta:\mathbb R^{C}\!\to\!\mathbb R^{D}$ and $E_\varphi:\mathbb R^{D}\!\to\!\mathbb R^{C}$ be feed-forward networks built from affine layers and piecewise-linear activations (e.g.\ ReLU, LeakyReLU).

Draw $Z\sim p(z)$ and denote the \emph{effective-support event}
\(
  \mathcal E_\zeta:=\{\lVert Z\rVert\le R_\zeta\},
\)
which satisfies $\Pr(\mathcal E_\zeta)\ge1-\zeta$.  
On $\mathcal E_\zeta$ there exists a radius $r>0$ such that the activation pattern of every piecewise-linear unit is fixed on $B(Z,r)$.  Hence $D_\theta$ and $E_\varphi\!\circ\!D_\theta$ restrict to affine maps on $B(Z,r)$ and are therefore $\mathcal C^{1}$ there.

Define the random variables
\[
  A:=\lVert J_D(Z)\rVert_\sigma,
  \qquad
  B:=\lVert J_{E\circ D}(Z)-I\rVert_\sigma,
\]
and assume $\mathbb E[A^{2}]<\infty$ and $\mathbb E[B^{2}]<\infty$.

Then for any $\varepsilon\in(0,1)$ there exist thresholds $\tau_r,\tau_l>0$ (depending on $\varepsilon,\varrho$) such that, if
\[
  \mathcal L_{\mathrm{DR}}(\theta,\varphi)\le\tau_r,
  \qquad
  \mathcal L_{\mathrm{LR}}(\theta,\varphi)\le\tau_l,
\]
the following holds with probability at least $1-\zeta-\varrho$ over $Z\sim p$:

There exist constants $L_f<\infty$ and $\eta<1$ such that
\[
  A\le L_f,
  \qquad
  B\le \eta,
\]
and consequently $D_\theta$ is bi-Lipschitz on the ball $B(Z,r)$ with constant
\[
  L \;=\;\max\!\bigl(L_f,\tfrac{1}{1-\eta}\bigr),
\]
i.e.\ for all $z_1,z_2\in B(Z,r)$,
\begin{equation*}
  (1-\varepsilon)\,\lVert z_1-z_2\rVert
  \leq
  \lVert D_\theta(z_1)-D_\theta(z_2)\rVert
  \leq
  (1+\varepsilon)\,\lVert z_1-z_2\rVert.
\end{equation*}
For the proof of this Theorem, see Appendix \ref{app:proof_hp-bilip-relaxed}. As a result, minimizing the proposed LR loss with the conventional DR loss can help avoid posterior collapse according to our definition. 
\end{theorem}

\paragraph{Applicability to Both Piecewise‐Linear and Smooth Activations.}
Theorem \ref{thm:hp-bilip-relaxed} holds for networks built with piecewise-linear activations (e.g.\ ReLU or LeakyReLU), where each activation region is exactly affine and thus $\mathcal C^1$ locally, and for networks using globally smooth $\mathcal C^1$ activations (e.g.\ SiLU, GELU, or Tanh), where the Jacobians $J_D$ and $J_{E\circ D}$ are continuous on all of $\mathbb R^C$.  In both cases, for any $\varrho,\epsilon\in(0,1)$ and sufficiently small DR and LR losses, there exist constants $L_f<\infty$, $\eta<1$ and a radius $r>0$ such that with probability at least $1-\varrho$, $D_\theta$ is locally bi‐Lipschitz on $B(Z,r)$ with bi‐Lipschitz constant 
\[
  L \;=\;\max\!\bigl(L_f,\tfrac{1}{1-\eta}\bigr).
\]
as in Theorem \ref{thm:hp-bilip-relaxed}.

\subsection{Implementation Details}
Additionally, we adopt objectives with a weight $\alpha$ on our proposed term(LR) to optimize the model. The total objectives are represented as:
\begin{align*}
    &\mathcal{L}_{total}(x) 
    = \beta D_{KL}\big( q_{\varphi}(z|x_i) || p(z) \big) \\
    &+\mathbb{E}_{q_{\varphi}(z|x_i)}{[-\log{p_{\theta}(x_i|z)}]}
    +\alpha\mathbb{E}_{p_{\theta}(x|z_i)}{[-\log{q_{\varphi}(z_i|x)}]},
\end{align*}
where $\beta$\cite{higgins2017beta} is the weight of the regularization loss. 
Because a latent variable usually has a simple distribution and a low dimension, latent reconstruction loss is easier to optimize than data reconstruction loss. The global optimum of latent reconstruction loss would exist in infinitely many areas. We suppose that most of them could be bad local minima of the total loss. 
Therefore, using an arbitrary weighting or naive training scheme can lead to local minima where only latent reconstruction is satisfied, especially in the early stages of training. 
To mitigate this risk, we use a warm-up strategy that gradually increasing $\alpha$ from 0 to a specific hyperparameter $\alpha_T$, as described in each experiment. i.e., $\alpha:=\alpha_t$ is used when the epoch index is $t \in \{0,...,T\}$.

\section{Experiments}

% ## 실험1: target 분포를 아는 분포(single Gaussain)로 설정하고 실험. D_{KL}(q_{\varphi}(z|x)||p_{\theta}(z|x)) 와 D_{KL}(q_{{\varphi}}(z|x)||p(z)) , D_{KL}(q_{{\varphi}}(z)||p(z)) 를 모두 계산해서 training 중 변화 관찰 (그래프로 그리기) 
% 실험 설정: source는 standard normal, target은 standard는 아니지만 normal이다. 이 때 VAE가 이를 학습하고자 한다. 둘다 gaussain 이니까 affine decoder, affine encoder면 충분하고 global optimum도 알고 있다. 그러나 여기서 model trainer는 그 사실을 모른다고 가정(일반적인 상황) 따라서 실제 실험의 encoder decoder는 deep model로 설정.

% global optimum 에서: 정답인 Affine map은 리니어 이기 때문에, true posterior는 명확하게 p_{\theta}(z|x)=N(mu_z|x, sigma_z|x), where mu_z|x=sigma_z|x W^T sigma_x^{-1}(x-b), sigma_z|x=(W^T sigma^{-1} W + I)^{-1} 로 고정된다. 그리고 q(z) 역시 인코더가 affine이고, data distribution(target)이 normal이라는 것을 알기 때문에 closed form으로 알 수 있음. 그러면 세 term 모두 계산 가능.

% 실제 모델에서 세 지표 어떻게 측정할지?: 
%* D_{KL}(q_{\varphi}(z|x)||p_{\theta}(z|x)) : 우리는 true posterior 알기 떄문에 구할 수 있음
%* D_{KL}(q_{{\varphi}}(z|x)||p(z)) : 그냥 늘 구할 수 있음 (reg loss)
%* D_{KL}(q_{{\varphi}}(z)||p(z)) : q_{{\varphi}}(z) 샘플링 해서 kde 등으로 근사

% 각 세가지 지표 training 끝날 때 마다 저장해둠
% 별도의 plot (x축이 beta, y축이 각 지표 값) 함수를 통해 세가지 지표 beta 마다 변하는 것 관찰

% 예상되는 결과: high beta일 때 collapse 발생. beta/alpha를 조정하며 실험 했을 때 trade-off 확인 가능. beta-VAE와 LRVAE 비교했을 때 trade-off curve의 차이점 발생 -> 더 좋아지는지 모르겠어도 일부 구간에서 장점 설명하기

% 제안한 메소드가 두 posterior collapse 줄이면서 inference 방해 안한다는 실험
% 기존 논문들과 실험적 비교

In our experiments, we focus on showing that posterior collapse can be controlled by proposed latent reconstruction loss. 

\begin{table}[ht]
    \centering
    %\small
    \setlength{\tabcolsep}{5.5pt}
    \begin{tabular}{lcccccc}
        \hline
        Model       & $\beta$ & $L$ & $\alpha_T$  &   AU & KL &  MI \\
        \hline
        VAE         & 1.0   & -     & -  & 0.13 & 7.10 & 3.84  \\
        SA-VAE      & 1.0   & -     & -  & 0.13 & 6.89 & 3.77  \\
        Lagging VAE & 1.0   & -     & -  & 0.16 & 7.27 & 3.79  \\
        $\beta$-VAE & 0.1   & -     & -  & 0.22 & 18.12 & 3.93 \\
        $\beta$-VAE & 0.2   & -     & -  & 0.19 & 13.78 & 3.84 \\
        $\beta$-VAE & 0.4   & -     & -  & 0.16 & 10.54 & 3.91 \\
        LIDVAE      & 1.0   & 0     & -  & 0.42 & 25.15 & 3.88 \\
        IL-LIDVAE   & 1.0   & 0.1   & -  & 0.80 & 49.54 & 3.92 \\
        IL-LIDVAE   & 1.0   & 0.2   & -  & 0.72 & 54.68 & 3.93 \\
        IL-LIDVAE   & 1.0   & 0.4   & -  & 0.72 & 67.35 & 3.92 \\
        \textbf{LRVAE(ours)} & 0.1   & -     & 1.0  & \textbf{0.95} & 45.23 & 3.86 \\
        \textbf{LRVAE(ours)} & 0.2   & -     & 1.0  & 0.93 & 33.36 & \textbf{3.99} \\
        \textbf{LRVAE(ours)} & 1.0   & -     & 1.0  & 0.41 & 12.41 & 3.77 \\
        \hline
    \end{tabular}
    \caption{Quantitative evaluation on fashionmnist dataset}
    \label{tab:metric_fmnist}
\end{table}

\begin{table}[ht]
    \centering
    %\small
    \setlength{\tabcolsep}{5.5pt}
    \begin{tabular}{lcccccc}
        \hline
        Model       & $\beta$ & $L$ & $\alpha_T$  &  AU &   KL&   MI \\
        \hline
        VAE         & 1.0   & -     & -  & 0.16 & 8.69 & 3.96  \\
        SA-VAE      & 1.0   & -     & -  & 0.19 & 9.26 & 3.91  \\
        Lagging VAE & 1.0   & -     & -  & 0.19 & 9.18 & 3.84  \\
        $\beta$-VAE & 0.1   & -     & -  & 0.68 & 32.40 & 3.90 \\
        $\beta$-VAE & 0.2   & -     & -  & 0.38 & 23.48 & \textbf{3.93} \\
        $\beta$-VAE & 0.4   & -     & -  & 0.33 & 16.52 & 3.92 \\
        LIDVAE      & 1.0   & 0     & -  & 0.94 & 50.70 & 3.90 \\
        IL-LIDVAE   & 1.0   & 0.1   & -  & 0.96 & 70.19 & 3.85 \\
        IL-LIDVAE   & 1.0   & 0.2   & -  & 0.88 & 69.59 & 3.90 \\
        IL-LIDVAE   & 1.0   & 0.4   & -  & 0.94 & 89.83 & 3.89 \\
        \textbf{LRVAE(ours)} & 0.1   & -      & 0.4 & \textbf{1.00} & 45.20 & 3.84  \\
        \textbf{LRVAE(ours)} & 0.2   & -      & 0.4 & 0.99 & 32.80 & 3.90 \\
        \textbf{LRVAE(ours)} & 1.0   & -      & 0.4 & 0.22 & 11.92 & 3.91 \\
        \textbf{LRVAE(ours)} & 1.0   & -      & 1.0 & 0.66 & 14.85 & 3.90 \\
        \hline
    \end{tabular}
    \caption{Quantitative evaluation on MNIST dataset}
    \label{tab:metric_mnist}
\end{table}

\begin{table}[ht]
    \centering
    %\small
    \setlength{\tabcolsep}{5.5pt}
    \begin{tabular}{lcccccc}
        \hline
        Model       & $\beta$ & $L$ & $\alpha_T$  &   AU&  KL &  MI \\
        \hline
        VAE         & 1.0   & -     & -  & 0.06 & 2.32 & 2.11   \\
        SA-VAE      & 1.0   & -     & -  & 0.09 & 2.80 & 2.45   \\
        Lagging VAE & 1.0   & -     & -  & 0.06 & 2.23 & 2.03   \\
        $\beta$-VAE & 0.1   & -     & -  & 0.85 & 30.05 & \textbf{3.93}  \\
        $\beta$-VAE & 0.2   & -     & -  & 0.52 & 18.24 & 3.86  \\
        $\beta$-VAE & 0.4   & -     & -  & 0.22 & 11.33 & 3.91  \\
        LIDVAE      & 1.0   & 0     & -  & \textbf{0.99} & 15.02 & 3.85  \\
        IL-LIDVAE   & 1.0   & 0.2   & -  & 0.93 & 26.24 & 3.88  \\
        IL-LIDVAE   & 1.0   & 0.8   & -  & 0.97 & 109.73 & \textbf{3.93} \\
        IL-LIDVAE   & 1.0   & 1.0   & -  & 0.45 & 120.17 & 3.88 \\
        \textbf{LRVAE(ours)} & 1.0   & -     & 0.1  & 0.13 & 3.68 & 2.76   \\
        \textbf{LRVAE(ours)} & 1.0   & -     & 0.4  & 0.77 & 6.18 & 3.26   \\
        \textbf{LRVAE(ours)} & 1.0   & -     & 1.0  & 0.86 & 7.84 & 3.32   \\
        \hline
    \end{tabular}
    \caption{Quantitative evaluation on Omniglot dataset}
    \label{tab:metric_omniglot}
\end{table}

\begin{table}[ht]
    \centering
    %\small
    \setlength{\tabcolsep}{5.5pt}
    \begin{tabular}{lcccccc}
        \hline
        Model       & $\beta$ & $L$ & $\alpha_T$  &   AU& KL &  MI \\
        \hline
        VAE         & 1.0   & -     & -  & 0.47 & 40.87 & 3.85  \\
        %VAE$_{c}$         & 1.0   & -     & 0.0  & 0.69 & 46.25 & 4.00 \\
        SA-VAE      & 1.0   & -     & -  & 0.67 & 45.16 & 3.88 \\
        Lagging VAE & 1.0   & -     & -  & 0.63 & 42.46 & 3.96   \\
        $\beta$-VAE & 0.1   & -     & -  & 0.99 & 137.20 & 3.90   \\
        $\beta$-VAE & 0.2   & -     & -  & 0.99 & 85.46 & 3.92   \\
        LIDVAE      & 1.0   & 0     & -  & \textbf{1.00} & 307.72 & \textbf{3.99}  \\
        IL-LIDVAE   & 1.0   & 0.2   & -  & \textbf{1.00} & 336.86 & 3.88  \\
        IL-LIDVAE   & 1.0   & 1.0   & -  & 0.97 & 458.54 & 3.84  \\
        \textbf{LRVAE(ours)} & 1.0   & -     & 0.8  & \textbf{1.00} & 63.89 & 3.92    \\
        \hline
    \end{tabular}
    \caption{Quantitative evaluation on CelebA dataset}
    \label{tab:metric_celeba}
\end{table}

First, we evaluate and compare several models using quantitative metrics(AU, KL, MI) \cite{wang2021posterior}\cite{kinoshita2023controlling}. Results are shown in Table \ref{tab:metric_fmnist},  \ref{tab:metric_mnist}, \ref{tab:metric_omniglot}, and \ref{tab:metric_celeba}. Models used in the comparison are VAE\cite{kingma2013auto}, SA-VAE\cite{kim2018semi}, Lagging VAE\cite{he2019lagging}, $\beta$-VAE\cite{higgins2017beta}, LIDVAE\cite{wang2021posterior}, and IL-LIDVAE\cite{kinoshita2023controlling}. 

$\beta$ denotes weight of regularization loss, while $\alpha_T$ denotes final weight of our proposed latent reconstruction loss which is warm-up from 0 to described value through linear increase from epoch 0 to end(100). $L$ denotes inverse Lipschitz value of IL-LIDVAE. Activation Unit(AU) denotes ratio of activated latent channels which is calculated as AU$=\sum_c^C{\mathds{1}\{{Cov}_{p(x)}(\mathbb{E}_{q(z|x)}[z_{(c)}])\geq\varepsilon\}}$, where $z_{(c)}=\{z_{i,c},...,z_{i,c}\}$ is the $c$-th dimension of the latent variable $z$ for the $N$ validation data points. $\mathds{1}$ denotes an indicator function. Since deactivated channels are almost completely unidentifiable (related to $\varepsilon$), AU serves as a meaningful metric for estimating the upper bound of latent identifiability. Bigger AU is better. KL denotes sample-wise KL divergence (KL$=D_{KL}(q_{\varphi}(z|x_i)||p(z))$). This KL value indicates average of how far away each latent distribution of sample point is from prior. Too small value means that latent variable is not identifiable with prior. On the other hand, too big value means poor regularization. As we claim, the ideal case is the posterior converges to prior but each latent distribution has moderate KL values. MI denotes mutual information between $x_i$ and $z_i$. (MI$=I(z,x)=\mathbb{E}_{x}\big[\mathbb{E}_{q(z|x)}[q(z|x)]\big]-\mathbb{E}_{x}\big[\mathbb{E}_{q(z|x)}[q(z)]\big]$). MI indicates correlation between input data and latent variable. 
These metrics have some limitation to measure every aspect of posterior collapse(especially local posterior collapse), we can see influence in some meaningful aspects from them.

In table \ref{tab:metric_fmnist}, we can see evidence that smaller $\beta$ of $\beta$-VAE contributes to reduce posterior collapse on fashionMNIST dataset\cite{xiao2017fashion}. AU and KL are increased by decreasing $\beta$. When the $\beta$ is 1.0(VAE), AU is 0.13. The AU increases to 0.16, 0.19, and 0.22 when $\beta$ is decreased to 0.4, 0.2, and 0.1. Thus, reducing $\beta$ is meaningful influence, but limited. Decreasing $\beta$ always increases AU. However, be careful that decreasing $\beta$ also induces poor regularization. LIDVAE and IL-LIDVAE impressively increase AU to 0.42, 0.80, and 0.72 and KL to 25.15, 49.54, 54.58, and 67.35. However these KL values may be too large for regularization. Whereas, our LRVAE also achieves high AU such as 0.95, 0.93, and 0.41 with KL 45.23, 33.36, and 12.41. Although there is no a particular appropriate value, we say that this is an evidence that our method has benefit to avoid posterior collapse with better regularization.  

We also evaluate on MNIST\cite{deng2012mnist} as shown in Table \ref{tab:metric_mnist}. Our proposed loss term with a certain importance($\alpha_T=0.4$) contributes to increase AU from 0.68, 0.38, and 0.16 to 1.00, 0.99, and 0.22 when $\beta=0.1, 0.2, 1.0$ respectively. We also can see that higher importance($\alpha_T=1.0$) increases AU again to 0.66 from 0.22.

We can see similar pattern in experiments on Omniglot dataset\cite{lake2015human} in Table \ref{tab:metric_omniglot}. AU of VAE is only 0.06 but this value increases to 0.86 by minimizing latent reconstruction loss($\alpha_T=1.0$). Reducing importance of regularization loss($\beta$) to 0.1 also achieves similar AU(0.85), but the KL value also increases extremely(2.32 to 30.05). Results of IL-LIDVAE also show huge KL values(109.73, 120.17) while KL of LRVAE is relatively stable(7.84). MI value gradually increases from 2.11 in VAE to 2.76, 3.26, and 3.32 where importance of LR($\alpha_T$) increases from 0 to 0.1, 0.4, and 1.0.

Table \ref{tab:metric_celeba} shows results on CelebA dataset\cite{liu2015deep}. Our proposed LR also contribute to increase AU to 1.00 from 0.47 under proper KL value(63.89). On the other hand, LIDVAE and IL-LIDVAE achieve similar AU with KL values which are over 300.

\begin{figure}
    \centering
    \begin{subfigure}{0.24\columnwidth}
        \includegraphics[width=\textwidth]{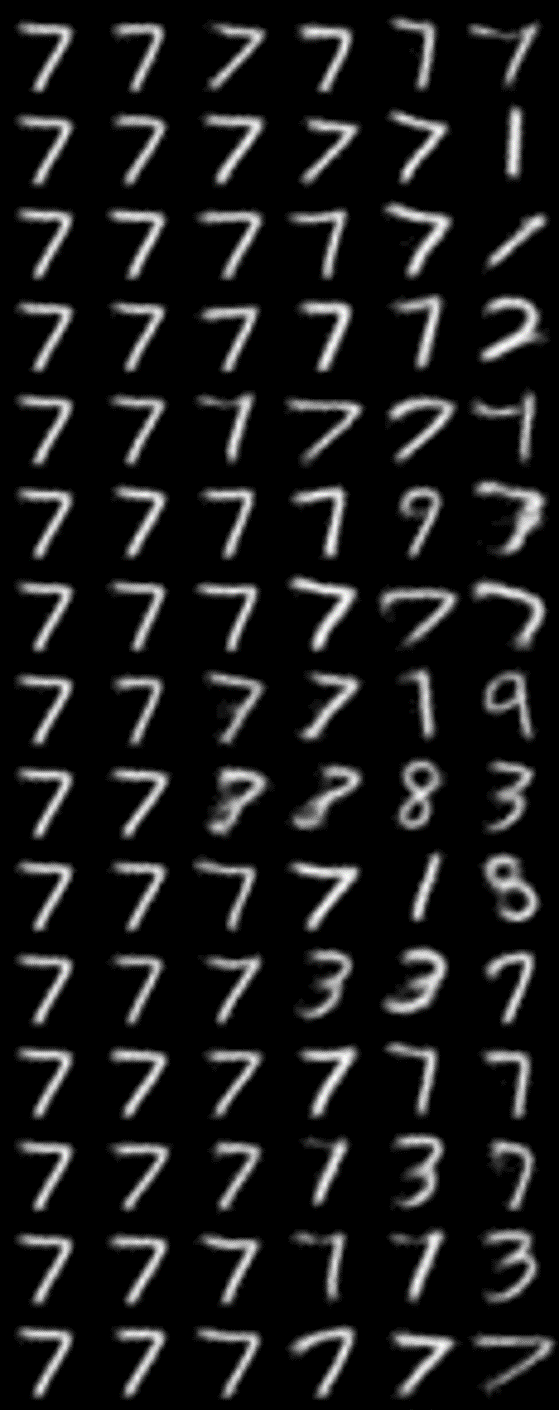}
        \caption{VAE}
        \label{fig:noise_mnist_convvae}
    \end{subfigure}
    \begin{subfigure}{0.24\columnwidth}
        \includegraphics[width=\textwidth]{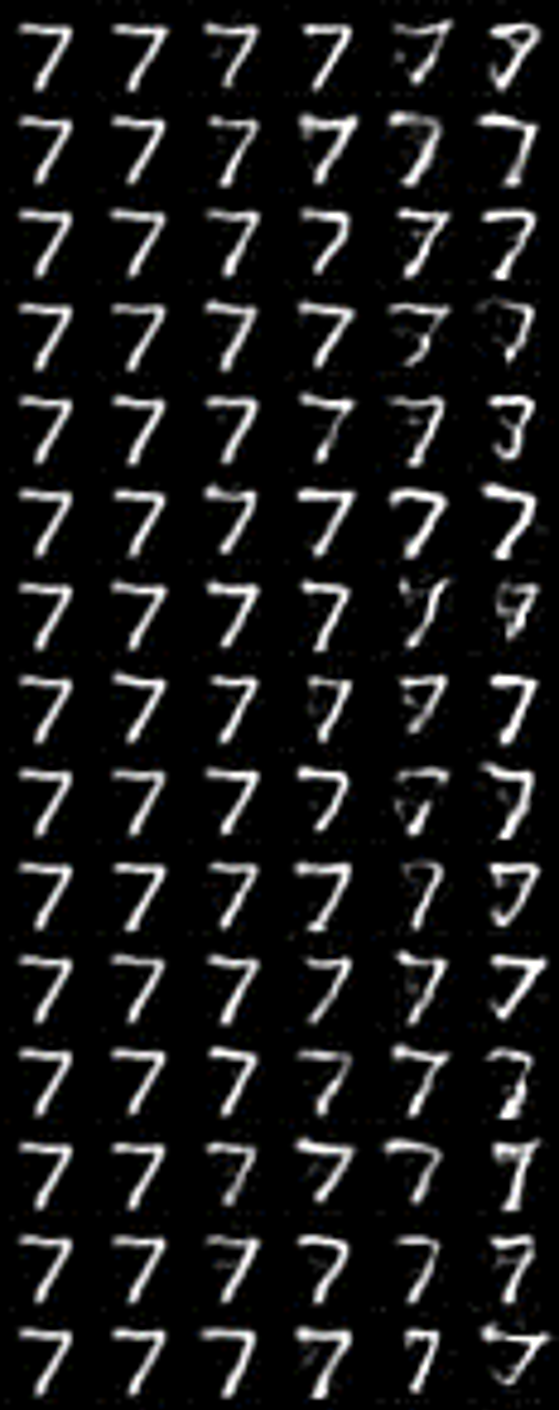}
        \caption{LIDVAE}
        \label{fig:noise_mnist_lidvae}
    \end{subfigure}
    \begin{subfigure}{0.24\columnwidth}
        \includegraphics[width=\textwidth]{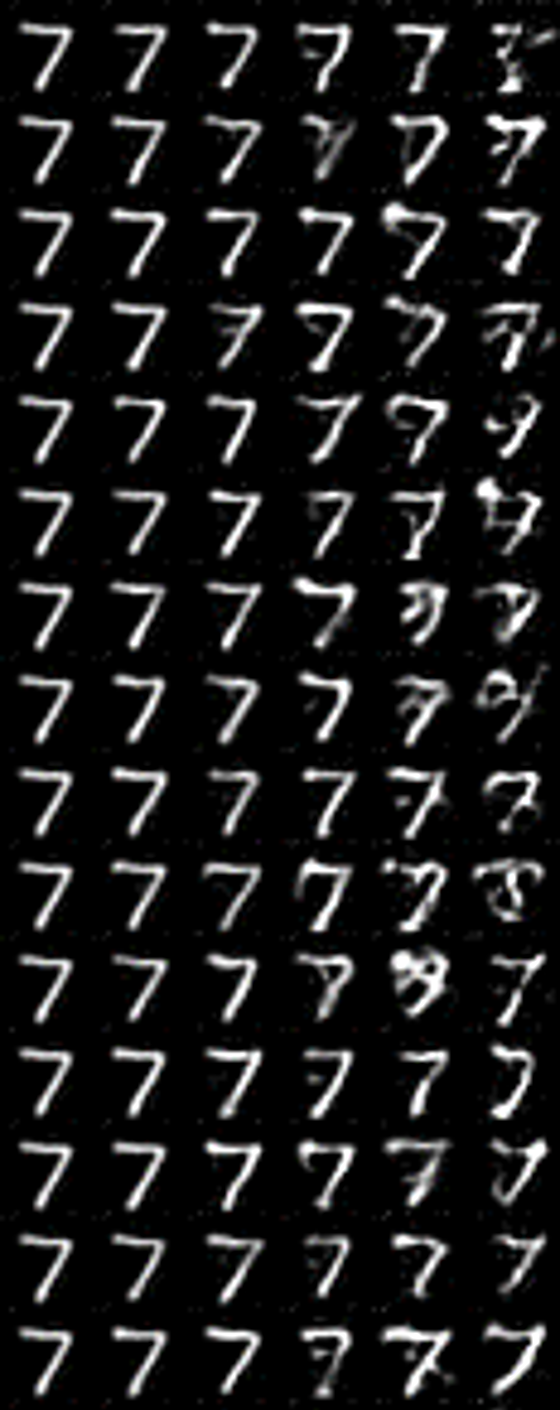}
        \caption{IL-LIDVAE}
        \label{fig:noise_mnist_illidvae}
    \end{subfigure}
    \begin{subfigure}{0.24\columnwidth}
        \includegraphics[width=\textwidth]{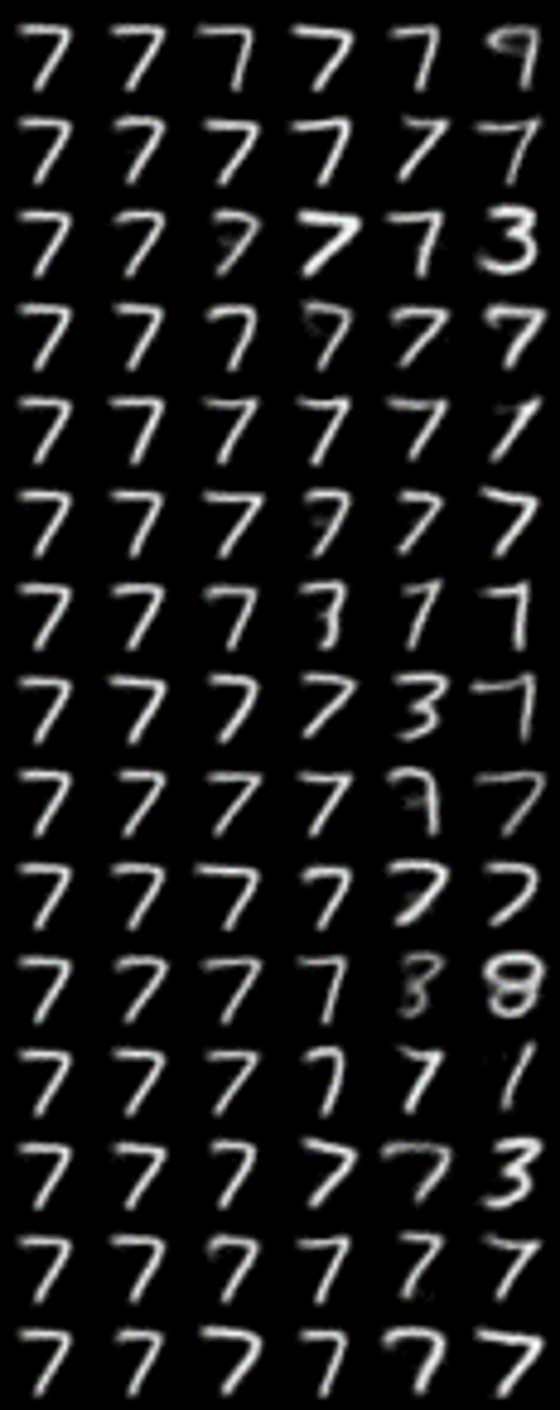}
        \caption{LRVAE}
        \label{fig:noise_mnist_convlrvae}
    \end{subfigure}
    \caption{Reconstruction from noised latent vectors results on MNIST dataset: IL-LIDVAE uses $L=0.2$ and LRVAE uses $\alpha_T=0.8$.}
    \label{fig:noise_mnist}
\end{figure}

Experiments are conducted to explore the response of decoders to local changes in latent variables as shown in Figure \ref{fig:noise_mnist}. In this figure, MNIST dataset\cite{deng2012mnist} is used. The columns of each sub-figure are the results of decoding the latent vector after adding Gaussian noise with variance of 0.0, 0.2, 0.4, 0.6, 0.8, and 1.0 from the left. That is, the leftmost column of each sub-figure is equal to the pure reconstruction result. In this experiment, all of results in the figure are reconstructed from an exactly same input data point(a handwriting `7'). In this experiment, we use shallow structure of only 1,168,641 parameters for VAE and LRVAE, while LIDVAE and IL-LIDVAE have 3,776,658 parameters. Models are trained 100 epochs equally. VAE(\ref{fig:noise_mnist_convvae}) generates out of class(1, 3, 8, or, 9) results compared other models. LIDVAE(\ref{fig:noise_mnist_lidvae}) and IL-LIDVAE(\ref{fig:noise_mnist_illidvae}) show tendency to produce images with added artifacts rather than normally diverse results. These artifacts are due to the side effect of the network restoration. 
Whereas, our LRVAE shows better results as shown in Figure \ref{fig:noise_mnist_convlrvae}.

\begin{figure}
    \centering
    \begin{subfigure}{0.47\columnwidth}
        \includegraphics[width=\textwidth]{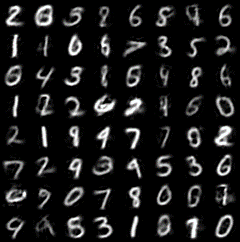}
        \caption{VAE}
        \label{fig:sampling_betavae}
    \end{subfigure}
    \begin{subfigure}{0.47\columnwidth}
        \includegraphics[width=\textwidth]{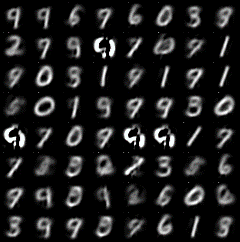}
        \caption{LIDVAE}
        \label{fig:sampling_lidvae}
    \end{subfigure}
    \begin{subfigure}{0.47\columnwidth}
        \includegraphics[width=\textwidth]{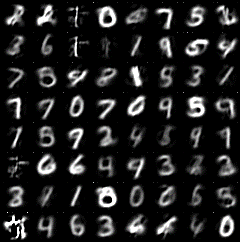}
        \caption{IL-LIDVAE($L=0.01$)}
        \label{fig:sampling_illidvae}
    \end{subfigure}
    \begin{subfigure}{0.47\columnwidth}
        \includegraphics[width=\textwidth]{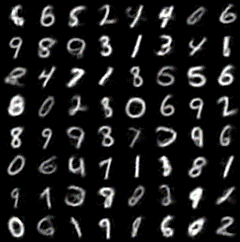}
        \caption{LRVAE(Ours, $\alpha_T=0.8$)}
        \label{fig:sampling_lrvae}
    \end{subfigure}
    \caption{Sampling(Generation) results on MNIST dataset}
    \label{fig:sampling}
\end{figure}

We provide generation results from randomly sampled latent vectors($z \sim N(0,I)$) in Figure \ref{fig:sampling}. In this experiment, LIDVAE and IL-LIDVAE have 3,776,658 parameters while $\beta$-VAE and LRVAE(ours) have 1,915,792 parameters. In our implementation, LIDVAE and IL-LIDVAE use single Gaussian prior for equal comparison with $\beta$-VAE and LRVAE, unlike original implementation\cite{wang2021posterior}\cite{kinoshita2023controlling} that uses GMVAE\cite{dilokthanakul2016deep}. We can see some collapsed outputs and less diversity in LIDVAE(Figure \ref{fig:sampling_lidvae}). Diversity can be improved by controlling inverse Lipschitz value $L$(Figure \ref{fig:sampling_illidvae}). However, collapsed outputs are still observed. Strictly restricted network structure has limitation to optimize mapping between complex data distribution and simple prior distribution. 

\begin{figure}[t!]
    \centering
    \begin{subfigure}{0.1\columnwidth}
        \includegraphics[width=\textwidth]{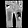}
        \caption{}
        \label{fig:noise_fmnist_origin}
    \end{subfigure}
    \begin{subfigure}{0.44\columnwidth}
        \includegraphics[width=\textwidth]{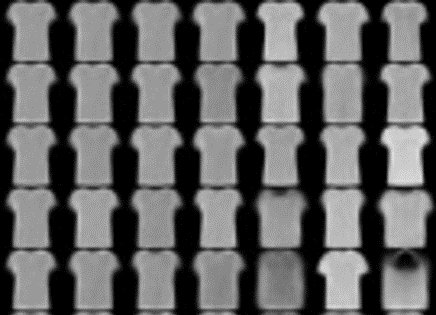}
        \caption{VAE}
        \label{fig:noise_fmnist_vae}
    \end{subfigure}
    \begin{subfigure}{0.44\columnwidth}
        \includegraphics[width=\textwidth]{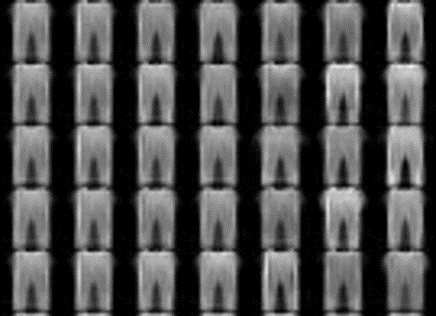}
        \caption{LIDVAE}
        \label{fig:noise_fmnist_lidvae}
    \end{subfigure}
    \raggedleft 
    \begin{subfigure}{0.44\columnwidth}
        \includegraphics[width=\textwidth]{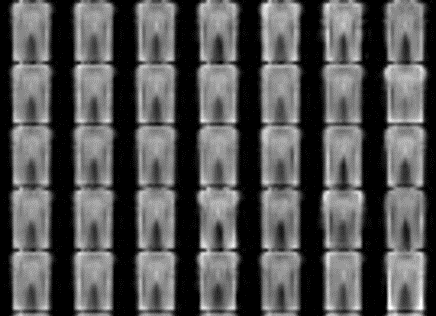}
        \caption{IL-LIDVAE}
        \label{fig:noise_fmnist_illidvae}
    \end{subfigure}
    \begin{subfigure}{0.44\columnwidth}
        \includegraphics[width=\textwidth]{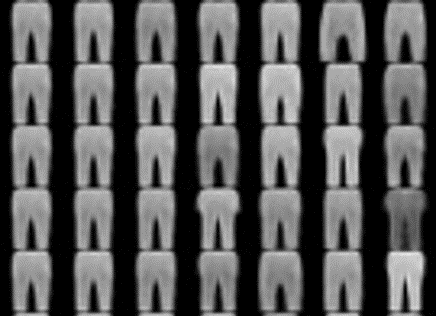}
        \caption{LRVAE (Ours)}
        \label{fig:noise_fmnist_convlrvae}
    \end{subfigure}
    \caption{Reconstruction from noised latent vectors on fashionMNIST dataset: (a) is the input image. IL-LIDVAE uses $L=0.2$ and LRVAE uses $\alpha_T=1.0$.}
    \label{fig:noise_fmnist}
\end{figure}
\begin{figure}[b!]
    \centering
    \begin{subfigure}{0.95\columnwidth}
        \includegraphics[width=\textwidth]{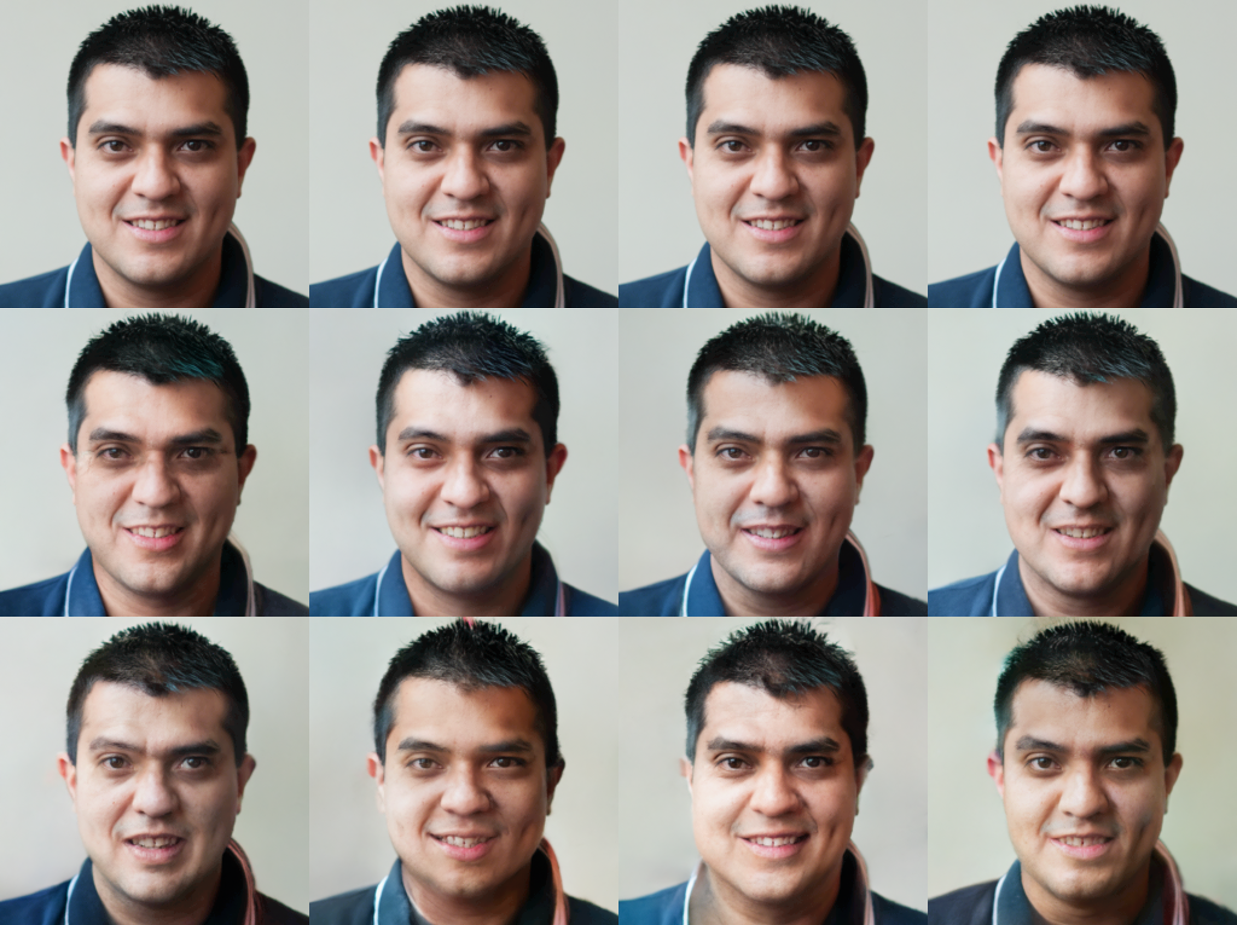}
        \caption{VDVAE}
        \label{fig:ffhq256_noise_vdvae}
    \end{subfigure}
    \hspace{0.5cm}
    \begin{subfigure}{0.95\columnwidth}
        \includegraphics[width=\textwidth]{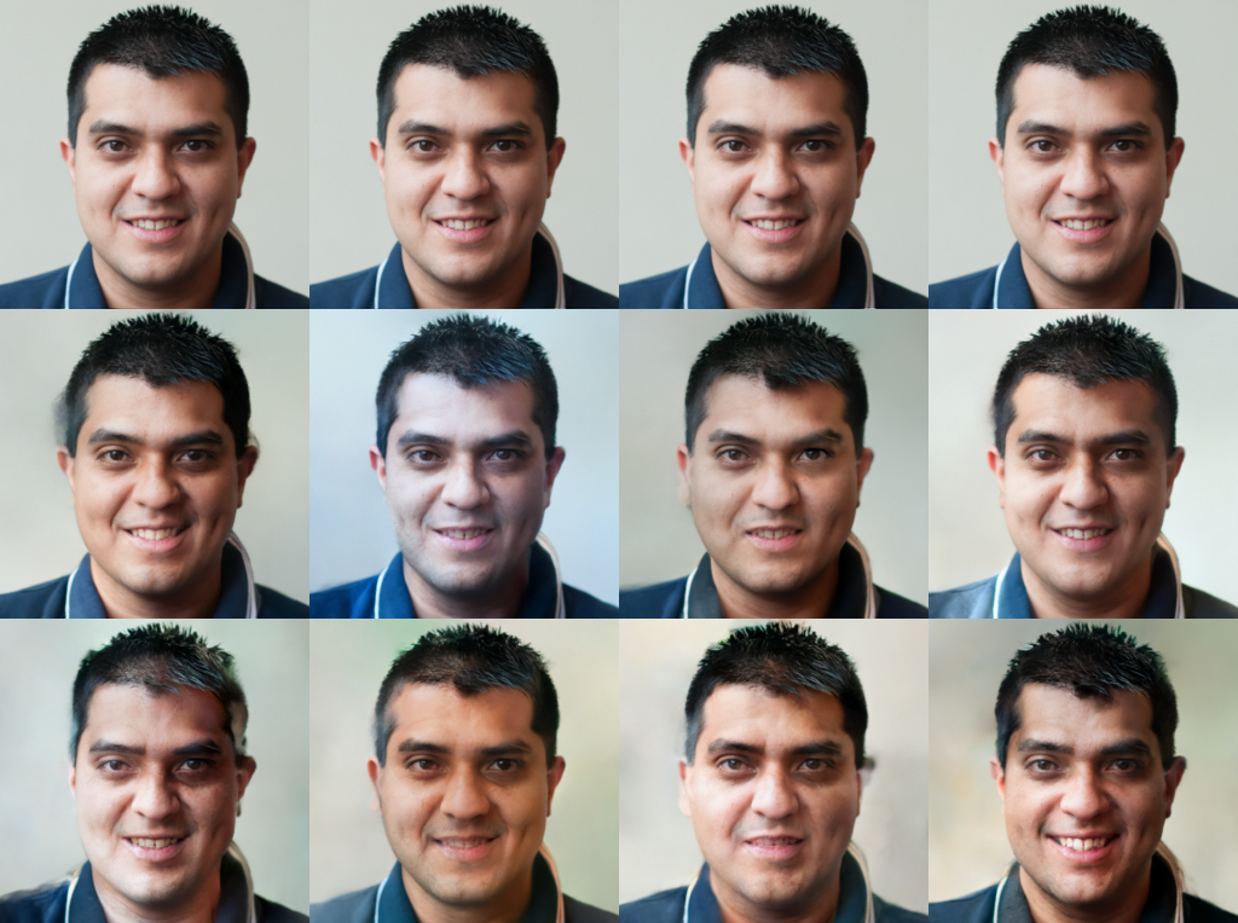}
        \caption{LRVDVAE (Ours)}
        \label{fig:ffhq256_noise_lrvdvae}
    \end{subfigure}
    \caption{Reconstruction results from noised latent vectors on FFHQ-256 dataset}
    \label{fig:ffhq256_noise}
\end{figure}

We also have reconstruction experiments with added tiny Gaussian noises on fashionMNIST\cite{xiao2017fashion} as shown in Figure \ref{fig:noise_fmnist}. VAE and LRVAE share the same network structure and parameters (1,168,641), while LIDVAE and IL-LIDVAE have identical network structures and parameters (3,776,658). Models are trained 100 epochs equally.  
%Figure \ref{fig:noise_fmnist} (a) 
%In this figure, we show input image  because pure reconstruction results(first column of (b), (c), (d), and (e)) also have somewhat differences with the input. 
VAE fails to reconstruct the pants images as shown in (b). The variance of Gaussian noise added increases by 0.5 from the left column to the right column. e.g., noises sampled from $N(0,3.0I)$ are added to the latent vector for last images. LIDVAE(c) and IL-LIDVAE(d) reconstruct pants, but their outputs are blurry and lack in diversity. Our LRVAE(e) successfully reconstructs the object of the input relatively clearly and showing diversity proportional to magnitude of noise.
We test on FFHQ-256\cite{karras2019style} dataset through VDVAE\cite{child2020very} for qualitative evaluation of our proposed approach on large scale and deeper model. We fine-tune the VDVAE from pre-trained parameters through additional 50,000 iteration with each different ELBO objectives. In Figure \ref{fig:ffhq256_noise}, LRVDVAE is trained using proposed ELBO objectives with $\alpha_T=1.0$, while VDVAE is trained using classical ELBO objectives. Each row shows images which are generated from latent vectors with the addition of four randomly sampled noise vectors from $N(0,0.0I)$, $N(0,0.3I)$, and $N(0,0.4I)$. i.e., the images at first row are equal to the pure reconstruction result. 

VDVAE(a) shows little difference among output images despite of added noise. On the other hand, LRVDVAE(b) shows larger variation among output images compared to VDVAE when the noise is sampled from $N(0,0.3I)$ and $N(0,0.4I)$.
%We hypothesize that the results are due to difference at local identifiability. 
Classical ELBO doesn't enforce the model to identify latent vectors which sampled from same input data. While, proposed ELBO enforces the model to response to small amount change of latent variable revealing improved local identifiability.

\section{Conclusion}
\label{sec:conclusion}

In this work, we introduce the concept of $\epsilon(x)$‑posterior collapse,  
a more general and relaxed notion that permits thresholds which vary locally across the data manifold.  
We further define $L(z)$‑bi‑Lipschitz continuity to express a locally elastic constraint on the decoder.  
We theoretically and empirically demonstrate that enforcing $L(z)$‑bi‑Lipschitz continuity helps mitigate $\epsilon(x)$‑posterior collapse. We believe this insight will broaden the practical utility of VAEs by removing unnecessary restrictions. % 문제 정의 의의
Moreover, we propose Latent Reconstruction(LR) loss as an additional loss to the conventional data reconstruction loss aimed at control posterior collapse by hyper-parameter $\alpha_T$. Minimizing LR loss achieves latent identifiability by leading the decoder and encoder to satisfy $L(z)$-bi-Lipschitz continuity. % 제안한 메소드
As a result, the LR loss enables controlling posterior collapse across most VAE architectures. In practice, our method enhances the diversity of the output in local regions while using fewer parameters. % 메소드의 결과 및 의의
In conclusion, our work contributes to a better understanding of posterior collapse and provides a pathway to address it under more practical settings for VAEs. %

% Limitations %alpha 값과 locall bi-Lip 상수가 비례하는 것은 아니다. 따라서 데이터와 네트워크에 따라 강한 alpha에서도 constraint가 매우 약하게 적용될 수 있다. 이러한 relaxation은 네트워크 제약에 따른 side effect를 줄이기 위한 우리가 의도한 바이지만, 어떤 경우에는 posterior collapse 회피 효과가 떨어질 수 있다.
%우리의 이론은 DR loss와 LR loss가 모두 충분히 낮아질 수 있다는 상황을 가정한다. 그러나 실제로는 두 loss를 모두 충분히 낮게 만들기는 쉽지 않을 수 있다. 두 loss를 모두 충분히 낮게 만들기 위해서는 완전한 가역함수는 아니더라도 어느정도의 가역성이 인코더와 디코더에 필요할 수 있을 것으로 추측된다. 본 연구에서는 이에 대해서 조사하지 않았다.
Despite these advancements, it's important to acknowledge certain limitations of this work. Our analysis indicates that the $\alpha$ value and the local bi-Lipschitz constant $L(z)$ would not be directly proportional. Consequently, for certain datasets and network architectures, a high $\alpha$ might still result in a weakly enforced constraint. While this relaxation is an intended design choice to mitigate side effects from network constraints, it could potentially diminish the effectiveness of preventing posterior collapse in some scenarios. Furthermore, our theoretical framework assumes that both the DR loss and LR loss can be sufficiently minimized. However, achieving sufficiently low values for both losses simultaneously can be challenging in practice. We hypothesize that a certain degree of invertibility, though not necessarily full invertibility, might be required for both the encoder and decoder to adequately minimize these losses. This aspect was not explored in the current study.

%\section*{Accessibility}
%\section*{Software and Data}

% Acknowledgements should only appear in the accepted version.
%\section*{Acknowledgements}

%\section*{Impact Statement}

\bibliography{lrvae}
\bibliographystyle{icml2025}

%%%%%%%%%%%%%%%%%%%%%%%%%%%%%%%%%%%%%%%%%%%%%%%%%%%%%%%%%%%%%%%%%%%%%%%%%%%%%%%
%%%%%%%%%%%%%%%%%%%%%%%%%%%%%%%%%%%%%%%%%%%%%%%%%%%%%%%%%%%%%%%%%%%%%%%%%%%%%%%
% APPENDIX
%%%%%%%%%%%%%%%%%%%%%%%%%%%%%%%%%%%%%%%%%%%%%%%%%%%%%%%%%%%%%%%%%%%%%%%%%%%%%%%
%%%%%%%%%%%%%%%%%%%%%%%%%%%%%%%%%%%%%%%%%%%%%%%%%%%%%%%%%%%%%%%%%%%%%%%%%%%%%%%
\newpage
\appendix
\onecolumn

\section{Proofs of Main Theorems}

\subsection{Proof of Theorem \ref{thm:LbL_avoid_LPC}} \label{app:proof_LbL_avoid_LPC}

\begin{theorem}[Gaussian VAE: $L(z)$-bi-Lipschitz avoids $\epsilon(x)$-posterior collapse with probability $1-\zeta$]\label{thm:gauss-laplace-delta}
Fix a confidence level $\zeta\in(0,1)$ and let radius $R_\zeta$ be defined as in Definition \ref{def:L-bi-Lip}.  
Let the prior be $p(z)=\mathcal N(0,I)$ and the likelihood
\[
  p(x\mid z)=\mathcal N\!\bigl(D_{\hat\theta}(z),\,\sigma^2 I\bigr).
\]
Assume the decoder $D_{\hat\theta}$ is $\mathcal C^2$ and \emph{$L(z)$-bi-Lipschitz with probability $1-\zeta$}.  
Consider any $x\in\mathcal X$ for which there exists $z_0$ satisfying
\[
  \|z_0\|\le R_\zeta, 
  \quad D_{\hat\theta}(z_0)=x.
\]
Define
\begin{align*}
  \Sigma &\;=\; \Bigl(I + \tfrac{1}{\sigma^2}\,J_D(z_0)^{\!\top} J_D(z_0)\Bigr)^{-1}, \\
  \kappa_{\min} &\;=\; \sigma_{\min}\bigl(J_D(z_0)\bigr)\;\ge\;\tfrac{1}{L(z_0)} \;>\;0.
\end{align*}
Then, under the Laplace approximation,
\begin{align*}
  &p(z\mid x)\approx \mathcal N\bigl(z_0,\Sigma\bigr) \\
  &\Rightarrow
  D_{\mathrm{KL}}\!\bigl(p(z\mid x)\,\|\,p(z)\bigr)
  \;\ge\;\frac{C}{2}\,
           \log\!\Bigl(1+\tfrac{\kappa_{\min}^{2}}{\sigma^{2}}\Bigr)
  \;>\;0.
\end{align*}
Consequently, with probability at least $1-\zeta$ over $x\sim p_{\mathrm{data}}$,
no $\epsilon(x)<\tfrac{C}{2}\log\!\bigl(1+\kappa_{\min}^{2}/\sigma^{2}\bigr)$
can satisfy Definition~\ref{def:epsx_posterior_collapse}.
\end{theorem}

\begin{proof}
We restate the calculation while explicitly conditioning on the
high--probability event that defines the \emph{effective support}.

\medskip\noindent\textbf{0. Conditioning on the effective support.}
Let
\(
  \mathcal E_\zeta := \{\|Z\|\le R_\zeta\}
\)
so that
\(
  \Pr_{Z\sim p}(\mathcal E_\zeta) \ge 1-\zeta
\)
by Definition \ref{def:L-bi-Lip}.  
Fix any \(x\in\mathcal X\) for which there exists
\(z_0\in\mathcal E_\zeta\) satisfying \(D_{\hat\theta}(z_0)=x\).
Because \(z_0\) lies in \(\mathcal E_\zeta\),
the local \(L(z)\)-bi-Lipschitz bound
\(
  \sigma_{\min}\!\bigl(J_D(z_0)\bigr) \ge 1/L(z_0) > 0
\)
is guaranteed.

\medskip\noindent\textbf{1. Laplace approximation of the posterior.}
Define
\(
  \ell(z)=\log p(x,z)=\log p(x\mid z)+\log p(z)
\).
A second--order Taylor expansion of \(-\ell(z)\) around its mode \(z_0\)
gives
\begin{align*}
  -\ell(z)
  &\approx \tfrac12\,(z-z_0)^\top
          \Bigl(I + \tfrac{1}{\sigma^{2}}\,
                J_D(z_0)^{\!\top} J_D(z_0)\Bigr)
          (z-z_0)
          + \text{const},
\end{align*}
whence
\(
  p(z\mid x)\approx\mathcal N(z_0,\Sigma)
\)
with
\(
  \Sigma^{-1}=I+\sigma^{-2}J_D(z_0)^{\!\top}J_D(z_0).
\)

\medskip\noindent\textbf{2. Exact KL divergence between two Gaussians.}
For any mean--covariance pair \((m,\Sigma)\),
\begin{equation*}
  D_{\mathrm{KL}}\!\bigl(
    \mathcal N(m,\Sigma)\;\|\;\mathcal N(0,I)
  \bigr)
  = \tfrac12\bigl(
      \mathrm{tr}(\Sigma)+m^\top m-\log\det\Sigma-C
    \bigr).
\end{equation*}
Substituting \(m=z_0\) and the above \(\Sigma\) yields
\begin{equation*}
  D_{\mathrm{KL}}\!\bigl(p(z\mid x)\,\|\,p(z)\bigr)
  \approx \tfrac12\bigl(
    \mathrm{tr}(\Sigma)+z_0^\top z_0-\log\det\Sigma-C
  \bigr).
\end{equation*}

\medskip\noindent\textbf{3. Lower bound via the smallest singular value.}
Because every eigenvalue of \(\Sigma^{-1}\) is at least \(1\),
each eigenvalue of \(\Sigma\) is at most \(1\), so
\(\mathrm{tr}(\Sigma)\le C\).
Dropping the non--negative term \(z_0^\top z_0\) gives
\begin{align*}
  D_{\mathrm{KL}}(p(z\mid x)\,\|\,p(z))
  &\ge -\tfrac12\log\det\Sigma
   =  \tfrac12\log\det\Sigma^{-1}.
\end{align*}
Let \(\varsigma_i\) be the singular values of \(J_D(z_0)\); then
the eigenvalues of \(\Sigma^{-1}\) are
\(
  \lambda_i = 1+\varsigma_i^{2}/\sigma^{2}
\),
so, with
\(
  \kappa_{\min}:=\min_i\varsigma_i\ge 1/L(z_0)>0
\),
\begin{align*}
  \det\!\bigl(\Sigma^{-1}\bigr)
    &=\prod_{i=1}^{C}\bigl(1+\varsigma_i^{2}/\sigma^{2}\bigr)
     \;\ge\;
     \bigl(1+\kappa_{\min}^{2}/\sigma^{2}\bigr)^{C},\\
  D_{\mathrm{KL}}(p(z\mid x)\,\|\,p(z))
    &\ge \tfrac{C}{2}\,
          \log\!\bigl(1+\kappa_{\min}^{2}/\sigma^{2}\bigr)
      \;>\;0.
\end{align*}

\medskip\noindent\textbf{4. Translating the bound to a probabilistic statement.}
Steps~0--3 hold for every \(z_0\in\mathcal E_\zeta\); hence
\[
  \Pr_{Z\sim p}\!\Bigl(
    D_{\mathrm{KL}}\bigl(p(z\mid x)\,\|\,p(z)\bigr)>0
  \Bigr)\;\ge\;1-\zeta.
\]
Consequently, with the same probability over \(x\) induced by the
generative process, no
\(
  \epsilon(x) < \tfrac{C}{2}\log\bigl(1+\kappa_{\min}^{2}/\sigma^{2}\bigr)
\)
can satisfy Definition~\ref{def:epsx_posterior_collapse}.
This completes the proof.
\end{proof}

\subsection{Proof of Theorem \ref{thm:hp-bilip-relaxed}} \label{app:proof_hp-bilip-relaxed}

\begin{theorem}[High–Probability Locally Bi-Lipschitz under Piecewise-Linear Networks]
Fix confidence levels $\zeta,\varrho\in(0,1)$ and let radius $R_\zeta$ be defined as in Definition \ref{def:L-bi-Lip}.  
Let $D_\theta:\mathbb R^{C}\!\to\!\mathbb R^{D}$ and $E_\varphi:\mathbb R^{D}\!\to\!\mathbb R^{C}$ be feed-forward networks built from affine layers and piecewise-linear activations (e.g.\ ReLU, LeakyReLU).

Draw $Z\sim p(z)$ and denote the \emph{effective-support event}
\(
  \mathcal E_\zeta:=\{\lVert Z\rVert\le R_\zeta\},
\)
which satisfies $\Pr(\mathcal E_\zeta)\ge1-\zeta$.  
On $\mathcal E_\zeta$ there exists a radius $r>0$ such that the activation pattern of every piecewise-linear unit is fixed on $B(Z,r)$.  Hence $D_\theta$ and $E_\varphi\!\circ\!D_\theta$ restrict to affine maps on $B(Z,r)$ and are therefore $\mathcal C^{1}$ there.

Define the random variables
\[
  A:=\lVert J_D(Z)\rVert_\sigma,
  \qquad
  B:=\lVert J_{E\circ D}(Z)-I\rVert_\sigma,
\]
and assume $\mathbb E[A^{2}]<\infty$ and $\mathbb E[B^{2}]<\infty$.

Then for any $\varepsilon\in(0,1)$ there exist thresholds $\tau_r,\tau_l>0$ (depending on $\varepsilon,\varrho$) such that, if
\[
  \mathcal L_{\mathrm{DR}}(\theta,\varphi)\le\tau_r,
  \qquad
  \mathcal L_{\mathrm{LR}}(\theta,\varphi)\le\tau_l,
\]
the following holds with probability at least $1-\zeta-\varrho$ over $Z\sim p$:

There exist constants $L_f<\infty$ and $\eta<1$ such that
\[
  A\le L_f,
  \qquad
  B\le \eta,
\]
and consequently $D_\theta$ is bi-Lipschitz on the ball $B(Z,r)$ with constant
\[
  L \;=\;\max\!\bigl(L_f,\tfrac{1}{1-\eta}\bigr),
\]
i.e.\ for all $z_1,z_2\in B(Z,r)$,
\[
  (1-\varepsilon)\,\lVert z_1-z_2\rVert
  \;\le\;
  \lVert D_\theta(z_1)-D_\theta(z_2)\rVert
  \;\le\;
  (1+\varepsilon)\,\lVert z_1-z_2\rVert.
\]
\end{theorem}

\begin{proof}
\textbf{0. Conditioning on the effective support.}
Work on the event $\mathcal E_\zeta$; this costs probability at most $\zeta$.

\medskip\noindent\textbf{1. Probabilistic bounds on $A,B$.}
Since $\mathbb E[A^{2}]$ and $\mathbb E[B^{2}]$ are finite, Markov’s inequality yields
constants $L_f<\infty$ and $\eta<1$ such that
\[
  \Pr\!\bigl(A>L_f\bigr)\le\tfrac{\varrho}{2},
  \qquad
  \Pr\!\bigl(B>\eta\bigr)\le\tfrac{\varrho}{2}.
\]
By a union bound, with probability at least $1-\varrho$ (unconditionally) we have
$A\le L_f$ and $B\le\eta$.

\medskip\noindent\textbf{2. Local affine region.}
For any $Z$ (in fact for all but a measure-zero set) there exists
$r>0$ such that the activation pattern of every piecewise-linear unit
is constant on $B(Z,r)$.  Consequently,
\[
  J_D(z)=J_D(Z),
  \qquad
  J_{E\circ D}(z)=J_{E\circ D}(Z),
  \qquad\forall z\in B(Z,r).
\]

\medskip\noindent\textbf{3. Lipschitz upper bound.}
Because $D_\theta$ is affine on $B(Z,r)$,
\[
  \lVert D_\theta(z_1)-D_\theta(z_2)\rVert
  =\lVert J_D(Z)(z_1-z_2)\rVert
  \le A\,\lVert z_1-z_2\rVert
  \le(L_f+\varepsilon)\lVert z_1-z_2\rVert.
\]

\medskip\noindent\textbf{4. Lipschitz lower bound.}
On the same ball, $\lVert J_{E\circ D}(Z)-I\rVert_\sigma=B\le\eta<1$,
so by the inverse-function theorem for affine maps,
\[
  \lVert D_\theta(z_1)-D_\theta(z_2)\rVert
  \ge(1-\eta)\,\lVert z_1-z_2\rVert.
\]

\medskip\noindent\textbf{5. Collecting probabilities.}
Steps~1–4 hold whenever (i) $\mathcal E_\zeta$ occurs and (ii) the event
$\{A\le L_f,\;B\le\eta\}$ occurs.  
By the union bound the joint probability is at least $1-\zeta-\varrho$,
yielding the desired two-sided inequality with
\(
  L=\max(L_f,1/(1-\eta))
\)
and completing the proof.
\end{proof}

\section{Definition List of Posterior Collapse}
\label{app:vary_pc}

\begin{table}[h]
\centering
\small
\caption{Comparison of posterior collapse definitions across different works.}
\label{tab:pc_definitions}
\begin{tabular}{lccccc}
\toprule
\textbf{Work} & \textbf{Definition} & \textbf{Condition} & \textbf{Metric} & \textbf{Locality} & \textbf{Target} \\
\midrule
\cite{wang2021posterior} & $p_{\hat{\theta}}(z|x) = p(z)$ & Equality & Exact match & Global & $p_{\hat{\theta}}$ \\
\cite{razavi2019preventing} & $D_{\mathrm{KL}}(q(z|x)\|p(z)) \to 0$ & Vanishing KL & KL divergence & Global & $q_{{\varphi}}$ \\
\cite{lucas2019don} & $\mathbb{P}_{x \sim p(x)}[D_{\mathrm{KL}} < \epsilon] \ge 1 - \Delta$ & Probabilistic & KL divergence & Global & $q_{{\varphi}}$ \\
\cite{kinoshita2023controlling} & $d(q_{\hat{\varphi}}(z|x), p(z)) \le \epsilon$ & Threshold & General divergence & Global & $q_{{\varphi}}$ \\
\textbf{Ours} & $d(\rho(z|x), p(z)) \le \epsilon(x)$ & Functional bound & General divergence & \textbf{Local} & $p_{\hat{\theta}}$, $q_{\hat{\varphi}}$ \\
\bottomrule
\end{tabular}
\end{table}

Table~\ref{tab:pc_definitions} summarizes several existing definitions of posterior collapse and compares them with our proposed $\epsilon(x)$-posterior collapse. The earliest work by \cite{wang2021posterior} adopts the strongest criterion, defining collapse as exact equality between the true posterior and prior, which excludes even small meaningful variations. Later works~\cite{razavi2019preventing, kinoshita2023controlling} propose divergence-based criteria such as a vanishing KL divergence or a constant threshold $\epsilon$, but still apply the condition uniformly over the input space.

\cite{lucas2019don} proposes a probabilistic relaxation, introducing $(\epsilon, \Delta)$-posterior collapse, which tolerates deviations on a small subset of the data distribution. However, all of these approaches remain global in their treatment of input samples.

In contrast, our formulation introduces a data-dependent threshold function $\epsilon(x)$ that varies across the domain $\mathcal{X}$. This allows a flexible treatment of regions with low informativeness (e.g., noise), while maintaining strict identifiability around data points. Additionally, by adopting the true posterior $p(z|x)$ as the target rather than the approximate posterior, our definition enables clearer theoretical analysis and model evaluation. The locality and adaptiveness of our formulation make it both a generalization and refinement of previous definitions.

\section{Symmetric loss and mutual information}

\begin{theorem}[Latent Reconstruction Loss Tightens a Mutual–Information Lower Bound]
\label{thm:lr-promotes-mi}
Let the generative joint be \(p(z,x)=p(z)\,p_\theta(x\mid z)\) with latent prior \(p(z)\).
For any variational family \(q_\varphi(z\mid x)\) define the \emph{latent–reconstruction loss}
\begin{equation*}
  \mathcal L_{\mathrm{LR}}(\theta,\varphi)
  := \mathbb E_{p(z,x)}\!\bigl[-\log q_\varphi(z\mid x)\bigr].
\end{equation*}
Then the following decomposition holds:
\begin{equation}\label{eq:app_LR}
  \boxed{\;
    \mathcal L_{\mathrm{LR}}(\theta,\varphi)
    \;=\;
    H(Z\mid X)
    + \mathbb E_{p(x)}\! \bigl[
        D_{\mathrm{KL}}\bigl(p(z\mid x)\;\|\;q_\varphi(z\mid x)\bigr)
      \bigr]
  \;}
\end{equation}
Consequently,
\begin{equation}
  \mathcal L_{\mathrm{LR}}(\theta,\varphi)
  \;\ge\;
  H(Z\mid X)
  \;=\;
  H(Z)-I(Z;X),
\end{equation}
with equality iff \(q_\varphi(z\mid x)=p(z\mid x)\) almost surely.
Hence minimizing \(\mathcal L_{\mathrm{LR}}\) (i) tightens the variational posterior toward the true one and (ii) \emph{maximizes a lower bound on the mutual information}:
\begin{equation}
  I(Z;X)
  \;\ge\;
  H(Z)-\mathcal L_{\mathrm{LR}}(\theta,\varphi).
\end{equation}
\end{theorem}

\begin{proof}
Let \(p(z,x)=p(z)\,p_\theta(x\mid z)\) and abbreviate \(q(z\mid x):=q_\varphi(z\mid x)\).
\begin{align*}
  \mathcal L_{\mathrm{LR}}
  &= \mathbb E_{p(z,x)}\!\bigl[-\log q(z\mid x)\bigr]
   = \mathbb E_{p(x)}
      \mathbb E_{p(z\mid x)}\!\bigl[-\log q(z\mid x)\bigr]            \\
  &= \mathbb E_{p(x)}\Bigl[
       H_{p,q}\!\bigl(Z\mid X{=}x\bigr)
     \Bigr]
   \quad
   \bigl(\text{cross entropy}\bigr)                                      \\
  &= \mathbb E_{p(x)}\!\bigl[
       H\!\bigl(Z\mid X{=}x\bigr)
       + D_{\mathrm{KL}}\!\bigl(p(z\mid x)\;\|\;q(z\mid x)\bigr)
     \bigr]                                                              \\
  &= H(Z\mid X)
     + \mathbb E_{p(x)}\!\bigl[
         D_{\mathrm{KL}}\!\bigl(p(z\mid x)\;\|\;q(z\mid x)\bigr)
       \bigr],
\end{align*}
which is identity (1).  
The KL term is non–negative, giving the inequality (2).  
Since \(H(Z\mid X)=H(Z)-I(Z;X)\), rearranging yields the lower bound (3).
The KL term vanishes exactly when \(q(z\mid x)=p(z\mid x)\) a.s., proving tightness.
\end{proof}

\medskip
Equation~\eqref{eq:app_LR} shows that the latent--reconstruction loss
naturally decomposes into two terms:
(i) the \emph{intrinsic} uncertainty $H(Z\mid X)$ imposed by the generative
model and (ii) an \emph{encoder‐mismatch penalty}
$\mathbb E_{p(x)}\bigl[D_{\mathrm{KL}}\bigl(p(z\mid x)\,\|\,q_\varphi(z\mid x)\bigr)\bigr]$.
Minimizing $\mathcal L_{\mathrm{LR}}$ therefore tightens $q_\varphi(z\mid x)$
toward the true posterior \emph{while} driving up the lower bound
$H(Z)-\mathcal L_{\mathrm{LR}}$ on the mutual information $I(Z;X)$.
In practical terms, LR loss promotes
\emph{information‐rich latents} without requiring explicit mutual‐information
estimators such as MINE.
When combined with the ELBO (or DR loss), LR acts as
a principled counterpart to the KL term:
it reduces posterior collapse by explicitly rewarding
cycle consistency $z\!\to\!x\!\to\!\hat z$ in latent space,
which the standard ELBO is blind to.
Empirically (Section~4, Fig.~5), increasing the LR coefficient raises measured
$I(Z;X)$ and narrows the gap to the theoretical bound,
demonstrating that the inequality in~\eqref{eq:app_LR} is not only
tight in theory (equality when $q_\varphi=p$) but also
\emph{practically tight} for well‐trained models.
This connection provides a clear design principle:
\[
  \text{Lower $\mathcal L_{\mathrm{LR}}$} 
  \;\;\Longrightarrow\;\;
  \text{Higher } I(Z;X),
\]
making LR loss a simple yet powerful lever for controlling the information
capacity of VAEs.

\section{Symmetric loss and invertibility}

\begin{theorem}
\label{thr:true_approx_conv}
Let \( p_\theta(x|z) \) be a likelihood model and \( q_\varphi(z|x) \) an approximate posterior (encoder), with latent prior \( p(z) \).  
Assume that the following losses are minimized for all \( x \in \mathcal{X} \):
\begin{align*}
\mathcal{L}_{\mathrm{DR}}(x) &:= \mathbb{E}_{q_\varphi(z|x)}[-\log p_\theta(x|z)], \\
\mathcal{L}_{\mathrm{LR}}(x) &:= \mathbb{E}_{p_\theta(x|z)}[-\log q_\varphi(z|x)].
\end{align*}
Then, under regularity conditions (smoothness, boundedness, and support alignment), the divergence between the approximate posterior and the true posterior vanishes:
\begin{equation*}
d(q_\varphi(z|x), p_\theta(z|x)) \to 0,
\end{equation*}
for suitable divergences \( d(\cdot,\cdot) \), such as the KL divergence or Wasserstein distance.
\end{theorem}

\begin{proof}
From Bayes’ rule, we write:
\begin{equation*}
p_\theta(z|x) = \frac{p_\theta(x|z) p(z)}{p_\theta(x)}.
\end{equation*}
Assume \( p(z) \) and \( p_\theta(x|z) \) are smooth and strictly positive, and that \( q_\varphi(z|x) \) has support overlapping with \( p_\theta(z|x) \).

\textbf{Step 1 (DR loss)}:  
Minimizing
\begin{equation*}
\mathcal{L}_{\mathrm{DR}}(x) = \mathbb{E}_{q_\varphi(z|x)}[-\log p_\theta(x|z)]
\end{equation*}
encourages \( q_\varphi(z|x) \) to assign high mass where \( p_\theta(x|z) \) is large, thus concentrating near modes of \( p_\theta(x|z) \).

\textbf{Step 2 (LR loss)}:  
Minimizing
\begin{equation*}
\mathcal{L}_{\mathrm{LR}}(x) = \mathbb{E}_{p_\theta(x|z)}[-\log q_\varphi(z|x)]
\end{equation*}
encourages \( q_\varphi(z|x) \) to be large in regions where \( p_\theta(x|z) \) is large, thus promoting matching with \( p_\theta(z|x) \).

\textbf{Step 3 (Combined loss)}:  
The symmetric loss is
\begin{equation*}
\mathcal{L}_{\mathrm{Sym}}(x) := \mathcal{L}_{\mathrm{DR}}(x) + \mathcal{L}_{\mathrm{LR}}(x),
\end{equation*}
which satisfies:
\begin{equation*}
\mathcal{L}_{\mathrm{Sym}}(x) = \mathbb{E}_{q_\varphi(z|x)}[-\log p_\theta(x|z)] + \mathbb{E}_{p_\theta(x|z)}[-\log q_\varphi(z|x)].
\end{equation*}
If this is minimized, and both expectations are finite, then the overlap between \( q_\varphi(z|x) \) and \( p_\theta(z|x) \) must be strong.

\textbf{Step 4 (Divergence bound)}:  
For instance, the KL divergence can be bounded:
\begin{equation*}
D_{\mathrm{KL}}(q_\varphi(z|x) \| p_\theta(z|x)) \le \mathcal{L}_{\mathrm{Sym}}(x) - H(p_\theta(z|x)),
\end{equation*}
where \( H(p_\theta(z|x)) \) is constant in \( \varphi \). Hence, minimizing \( \mathcal{L}_{\mathrm{Sym}} \) leads to small divergence.

Therefore, minimizing both \( \mathcal{L}_{\mathrm{DR}} \) and \( \mathcal{L}_{\mathrm{LR}} \) encourages \( q_\varphi(z|x) \to p_\theta(z|x) \), implying the approximation quality improves and posterior collapse is unlikely to occur.
\end{proof}

An accurate approximation of $q_\varphi(z|x) \to p_\theta(z|x)$ also encourages invertibility. We also show the invertibility from a function point of view as follows.
\begin{theorem}[Symmetric Reconstruction Loss Promotes Local Invertibility]
\label{thm:symmetric-loss-invertibility}
Let the encoder $E_\varphi: \mathcal{X} \to \mathcal{Z}$ and decoder $D_\theta: \mathcal{Z} \to \mathcal{X}$ define conditional distributions $q_\varphi(z|x)$ and $p_\theta(x|z)$ with Lipschitz-continuous means and uniformly bounded variances. We refer to the combination of data reconstruction loss and latent reconstruction loss as the symmetric reconstruction loss:
\begin{align*}
&\mathcal{L}_{Sym}(x)
:= \mathcal{L}_{DR}(x) + \mathcal{L}_{LR}(x) \\
&= \mathbb{E}_{q_\varphi(z|x)}[-\log p_\theta(x|z)] 
+ \mathbb{E}_{p_\theta(x|z)}[-\log q_\varphi(z|x)].
\end{align*}
Then, minimizing $\mathcal{L}_{Sym}(x)$ encourages the decoder $D_\theta$ and encoder $E_\varphi$ to satisfy locally inverse Lipschitz continuity:
\begin{align*}
\|z_i - z_j\| &\leq L(z_i, z_j) \cdot \|D_\theta(z_i) - D_\theta(z_j)\|, \\
\|x_i - x_j\| &\leq L(x_i, x_j) \cdot \|E_\varphi(x_i) - E_\varphi(x_j)\|,
\end{align*}
for some functions $L: \mathcal{Z} \times \mathcal{Z} \to \mathbb{R}_+$ and $L: \mathcal{X} \times \mathcal{X} \to \mathbb{R}_+$.
\end{theorem}

\begin{proof}
We analyze each component of $\mathcal{L}_{Sym}$.

\textbf{(1) Data reconstruction loss \(\mathcal{L}_{DR}(x)\):}
\begin{equation*}
\mathcal{L}_{DR}(x) = \mathbb{E}_{q_\varphi(z|x)}[-\log p_\theta(x|z)].
\end{equation*}
This loss is minimized when $x \approx D_\theta(z)$ for $z \sim q_\varphi(z|x)$. Under bounded variance, we approximate $z \approx \mu_\varphi(x)$, so $x \approx D_\theta(\mu_\varphi(x))$.

Take two distinct inputs $x_i, x_j$, and define $z_i := \mu_\varphi(x_i)$, $z_j := \mu_\varphi(x_j)$. Good reconstruction requires
\begin{equation*}
D_\theta(z_i) \approx x_i, \quad D_\theta(z_j) \approx x_j.
\end{equation*}
Thus, if $z_i \neq z_j$ but $D_\theta(z_i) \approx D_\theta(z_j)$, then $x_i \approx x_j$, contradicting the assumption. Therefore, to maintain accurate reconstructions, the decoder must satisfy:
\begin{equation*}
\|z_i - z_j\| \leq L(z_i, z_j) \cdot \|D_\theta(z_i) - D_\theta(z_j)\|
\end{equation*}
for some function \( L(z_i, z_j) \in \mathbb{R}_+ \). That is, $D_\theta$ is locally inverse Lipschitz continuous.

\textbf{(2) Latent reconstruction loss \(\mathcal{L}_{LR}(x)\):}
\begin{equation*}
\mathcal{L}_{LR}(x) = \mathbb{E}_{p_\theta(x|z)}[-\log q_\varphi(z|x)].
\end{equation*}
This loss is minimized when $z \approx \mu_\varphi(x)$ for $x \sim p_\theta(x|z)$. Since $x \approx D_\theta(z)$ under bounded variance, define $x_i := D_\theta(z_i)$, $x_j := D_\theta(z_j)$.

To minimize the loss, we must have:
\begin{equation*}
E_\varphi(x_i) \approx z_i, \quad E_\varphi(x_j) \approx z_j.
\end{equation*}
Thus, if $\|z_i - z_j\|$ is large but $\|E_\varphi(x_i) - E_\varphi(x_j)\|$ is small, $q_\varphi(z|x)$ becomes ambiguous and the loss increases.

To avoid this, the encoder must satisfy:
\begin{equation*}
\|x_i - x_j\| \leq L(x_i, x_j) \cdot \|E_\varphi(x_i) - E_\varphi(x_j)\|
\end{equation*}
for some function \( L(x_i, x_j) \in \mathbb{R}_+ \). That is, $E_\varphi$ is locally inverse Lipschitz continuous.

Minimizing $\mathcal{L}_{Sym}$ ensures that both encoder and decoder avoid collapsing distinct inputs to similar outputs. This enforces local inverse Lipschitz continuity in both directions.
\end{proof}

\section{Notation}

Table~\ref{tab:notation} collects the main symbols and conventions used throughout this paper. 
Unless stated otherwise, uppercase letters denote random variables, lowercase letters denote deterministic quantities, and $\|\cdot\|_\sigma$ designates the matrix spectral norm. 
All expectations are taken with respect to the indicated distribution.

% Appendix Notation Table ----------------------------------
\begin{table}[H]
  \centering
  \renewcommand{\arraystretch}{1.15} % 행간 살짝 확장 (선택)
  \caption{Summary of notation}
  \label{tab:notation}
  \begin{tabular}{p{3.2cm} p{8cm}}
    \toprule
    \textbf{Symbol} & \textbf{Meaning / role} \\
    \midrule
    $x \in \mathcal X$ & Input sample (data space) \\
    $z \in \mathcal Z$ & Latent variable \\
    $p(z)$ & Prior over $z$ (standard normal) \\
    $q_{\phi}(z\!\mid\!x)$ & Encoder / approximate posterior \\
    $D_{\theta}(z)$ & Decoder network \\
    $E_{\phi}(x)$ & Deterministic encoder branch \\
    $J_D(z)$ & Jacobian of $D_{\theta}$ at $z$ \\
    $J_{E\!\circ\!D}(z)$ & Jacobian of $E_{\phi}\!\circ\!D_{\theta}$ at $z$ \\
    $A=\lVert J_D(Z)\rVert_{\sigma}$ & Spectral norm of decoder Jacobian (r.v.) \\
    $B=\lVert J_{E\circ D}(Z)-I\rVert_{\sigma}$ & Spectral deviation of composite Jacobian (r.v.) \\
    $L_{DR}$ & Data reconstruction loss \\
    $L_{LR}$ & Latent reconstruction loss \\
    $\tau_r,\;\tau_{\ell}$ & Loss thresholds s.t.\ $L_{DR}\!\le\!\tau_r$, $L_{LR}\!\le\!\tau_{\ell}$ \\
    $\zeta\!\in(0,1)$ & Confidence level for effective-support event $E_{\zeta}$ \\
    $\varrho\!\in(0,1)$ & Auxiliary confidence level (Markov bound) \\
    $R_{\zeta}$ & Radius with $\Pr(\lVert Z\rVert\le R_{\zeta})\!\ge\!1-\zeta$ \\
    $E_{\zeta}$ & Effective-support event $\{\lVert Z\rVert\le R_{\zeta}\}$ \\
    $r$ & Local radius (fixed activation pattern) \\
    $L_f$ & Upper bound on $\lVert J_D(Z)\rVert_{\sigma}$ in $E_{\zeta}$ \\
    $\eta$ & Upper bound on $\lVert J_{E\circ D}(Z)-I\rVert_{\sigma}$ in $E_{\zeta}$ \\
    $L(z)$ & Point-wise(local) bi-Lipschitz constant of $D_{\theta}$ \\
    $\epsilon(x)$ & Threshold in $\epsilon(x)$-posterior collapse definition \\
    \bottomrule
  \end{tabular}
\end{table}
% ----------------------------------------------------------

\end{document}